\documentclass[journal]{IEEEtran}
\usepackage[numbers,sort&compress]{natbib}
\usepackage{array}
\usepackage{amsfonts,amsmath,amsthm}
\usepackage{subfigure,graphicx}
\usepackage{algorithm}
\usepackage{algorithmicx}
\usepackage{algpseudocode}
\usepackage{epstopdf}
\theoremstyle{definition}
\newtheorem{thm}{Theorem}
\newtheorem{lem}{Lemma}
\newtheorem{defe}{Definition}

\newtheorem{cor}{Corollary}
\newtheorem{pro}{Proposition}
\newcommand{\RNum}[1]{\expandafter{\romannumeral #1\relax}}
\floatname{algorithm}{Algorithm}

\def\baselinestretch{1}
\begin{document}

\title{Impact of Prior Knowledge and Data Correlation on Privacy Leakage: A Unified Analysis}

\author{\IEEEauthorblockN{Yanan Li, Xuebin Ren, Shusen Yang, and Xinyu Yang}
\IEEEcompsocitemizethanks{
\IEEEcompsocthanksitem This work is supported in part by National Natural Science Foundation of China under Grants 61572398, 61772410, 61802298 and U1811461; the Fundamental Research Funds for the Central Universities under Grant xjj2018237; China Postdoctoral Science Foundation under Grant 2017M623177; the China 1000 Young Talents Program; and the Young Talent Support Plan of Xi'an Jiaotong University.  (corresponding author: Shusen Yang).
\IEEEcompsocthanksitem Y. Li is with National Engineering Laboratory for Big Data Analytics (NEL-BDA), Xi'an Jiaotong University, Xi'an, Shaanxi 710049, China, and also with the School of Mathematics and Statistics, Xi'an Jiaotong University, Xi'an, Shaanxi 710049, China (e-mail: gogll2@stu.xjtu.edu.cn).
\IEEEcompsocthanksitem X. Ren, and X. Yang are with the School of Electronic and Information Engineering, Xi'an Jiaotong University, Xi'an, Shaanxi 710049, China, and also with National Engineering Laboratory for Big Data Analytics (NEL-BDA), Xi'an Jiaotong University, Xi'an, Shaanxi 710049, China, (e-mails: \{xuebinren, yxyphd\}@mail.xjtu.edu.cn).
\IEEEcompsocthanksitem S. Yang is with National Engineering Laboratory for Big Data Analytics (NEL-BDA), Xi'an Jiaotong University, Xi'an, Shaanxi 710049, China, and also with the Ministry of Education Key Lab for Intelligent Networks and Network Security (MOE KLINNS Lab), Xi'an Jiaotong University, Xi'an, Shaanxi 710049, China (e-mail: shusenyang@mail.xjtu.edu.cn).
}}


\maketitle

\begin{abstract}
It has been widely understood that differential privacy (DP) can guarantee rigorous privacy against adversaries with arbitrary prior knowledge. However, recent studies demonstrate that this may not be true for correlated data, and indicate that three factors could influence privacy leakage: the data correlation pattern, prior knowledge of adversaries, and sensitivity of the query function. This poses a fundamental problem: what is the mathematical relationship between the three factors and privacy leakage?
In this paper, we present a unified analysis of this problem.
A new privacy definition, named \textit{prior differential privacy (PDP)}, is proposed to evaluate privacy leakage considering the exact prior knowledge possessed by the adversary. We use two models, the weighted hierarchical graph (WHG) and the multivariate Gaussian model to analyze discrete and continuous data, respectively.
We demonstrate that positive, negative, and hybrid correlations have distinct impacts on privacy leakage. Considering general correlations, a closed-form expression of privacy leakage is derived for continuous data, and a chain rule is presented for discrete data. Our results are valid for general linear queries, including count, sum, mean, and histogram. Numerical experiments are presented to verify our theoretical analysis.
\end{abstract}

\begin{IEEEkeywords}
privacy leakage, correlated data, prior knowledge.
\end{IEEEkeywords}

\IEEEpeerreviewmaketitle

\section{Introduction}
\label{Sec:introduciton}

 Leakage of private information could lead to serious consequences (e.g., financial security and personal safety), and privacy protection has been extensively studied for several decades \cite{Dalenius1977Towards,Cox1980Suppression}. In today's big data era, privacy issues have been attracting increasing attention from both society and academia~\cite{Juels2006RFID,Fang2012Smart,yang2017survey,voigt2017eu}.
 Differential privacy (DP) \cite{dwork2014algorithmic,dwork2008differential,Dwork2006Calibrating} has become the defacto standard for privacy definitions because it can provide a rigorously mathematical proof of privacy guarantees.

In practice, adversaries may be able to acquire prior knowledge (i.e., partial data records), due to database attacks \cite{Guarnieri2016Strong}, privacy incidents \cite{Li2013Membership}, and obligations to release~\cite{Kifer2014Pufferfish}.
 It is commonly believed that differentially private algorithms are invulnerable to adversaries with arbitrary prior knowledge because any given privacy level can be guaranteed, even when the adversary has knowledge of all data records except certain ones (i.e., the adversary with the strongest prior knowledge).
 However, this is true only if all data records are independent.
 It has been shown that the adversary's prior knowledge can have significant impacts on privacy leakage when data records are correlated~\cite{Kifer2011No,Kifer:2012:RCF:2213556.2213571}.

 The following example demonstrates how privacy leakage can be affected by correlations and the adversaries' prior knowledge.


 \begin{figure}
  \centering
  \includegraphics[width=9cm]{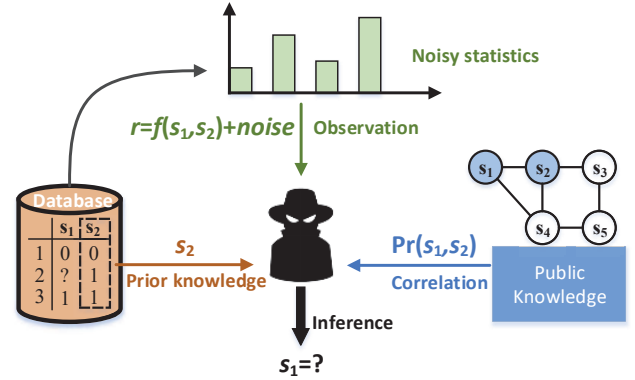}\\
  \caption{Illustration of Example 1: an adversary attempts to infer the information of $s_1$ based on the joint distribution of database $\mathbf{x}$, the published result $r$, and his prior knowledge about $s_2$.}\label{fig:inferencemodel}
\end{figure}

\begin{figure*}
  \centering
  \includegraphics[width=18cm]{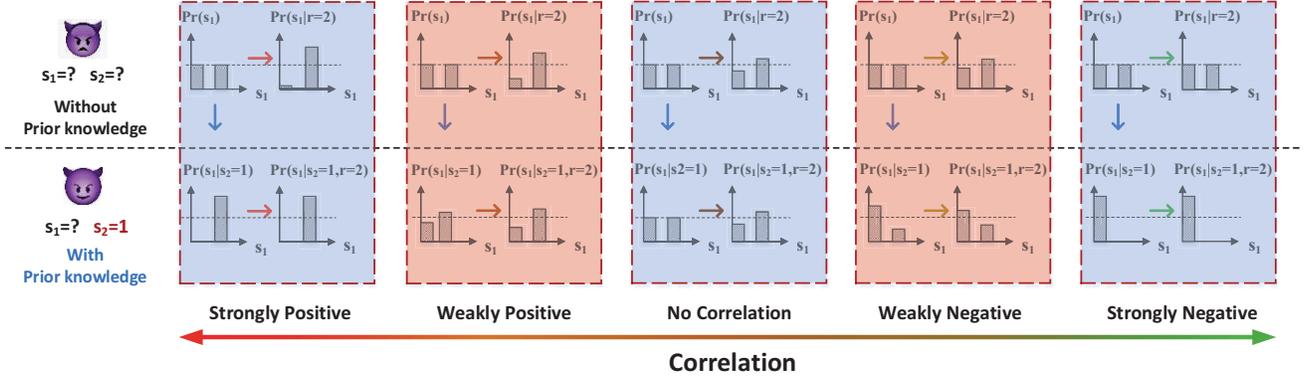}\\
  \caption{Illustration of Example 1: the inference results of two adversaries; the weak adversary knows nothing about $s_2$, the strong adversary knows $s_2=1$. Considering five correlations of $s_1$ and $s_2$, when the correlation is a perfectly positive correlation or perfectly negative correlation and independent, adversaries infer different information from the output result $r=2$. The problem is how to analyze the general impacts of the correlation and prior knowledge on privacy leakage.}\label{fig:inferenceresults}
\end{figure*}

 Example 1 \textit{Fig.~\ref{fig:inferencemodel} shows a scenario in which an adversary attempts to infer some sensitive information about a database. As shown, the database $\mathbf{x}$ consisting of two attributes, $s_1\in\{0,1\}$, and $s_2\in\{0,1\}$, publishes noisy (via a Laplace mechanism of differential privacy) statistics for privacy-preserving data mining. The adversary may acquire some prior knowledge about the database, i.e., the exact value of $s_1$ and the data correlations $Pr(s_1, s_2)$ from some public knowledge (e.g., the Internet). After observing the noisy statistics $r=f(s_1, s_2)+noise$, the adversary tries to infer the privacy of $s_1$ based on all available information. Assume a noisy statistic $r=s_1+s_2+noise=2$, the prior knowledge $s_2=1$, and the adversary's first impression about $s_1$ is $Pr(s_1=1)=Pr(s_1=0)=0.5$ before inference. The privacy information gain obtained by the adversary in the inference process is summarized in Fig.~\ref{fig:inferenceresults}. We use the following three special cases to show the impacts of the correlations and prior knowledge on privacy leakage.
 \begin{enumerate}
   \item Case 1 (Positive Correlation): $s_1$ and $s_2$ are perfectly positively correlated with coefficient $1$, i.e., $s_1=s_2$. Without prior knowledge, the adversary will infer $s_1+s_2=2$ from the observation with high confidence according to the characteristics of the Laplace mechanism. Combined with the correlation $s_1=s_2$, he will infer $s_1=1$ with high confidence. With the prior knowledge, e.g., $s_2=1$, the adversary can ascertain that $s_1=1$ from the correlation $s_1=s_2$. 
   \item Case 2 (Negative Correlation): $s_1$ and $s_2$ are perfectly negatively correlated with coefficient $-1$, i.e., $s_1+s_2=1$. Without prior knowledge, the adversary can infer no additional information about $s_1$ through $r=2$ due to the negative correlation. However, with the prior knowledge $s_2=1$, the adversary can claim that $s_1=0$. In addition, $r=2$ provides no additional information. 
   \item Case 3 (No Correlation): $s_1$ and $s_2$ are independent. Without prior knowledge, the adversary can infer that $s_1=1$ with relatively higher probability than $s_1=0$ from the observation $r=2$. However, with the additional prior knowledge of $s_2=1$, the adversary obtains no more confidence about $s_1$ because there is no correlation between $s_1$ and $s_2$, i.e., a stronger adversary with extra prior knowledge achieves no privacy gain compared with a weaker adversary.
 \end{enumerate}
The above special correlation cases show that an adversary with certain prior knowledge can obtain different privacy gains under different types of correlations.
For general correlation cases, i.e., when correlations are weakly positive or weakly negative (cases with red backgrounds in Fig.~\ref{fig:inferenceresults}), the adversary can also infer additional information through the published results.
Meanwhile, when correlations are perfectly positive or negative, adversaries with different prior knowledge can also gain different privacy information.}



As demonstrated in the above examples, prior knowledge can be utilized by adversaries to infer sensitive information, leading to serious threats to various privacy preserving scenarios, such as data publishing~\cite{zhu2017differentially,zhang2017privbayes,kulkarni2017marginal,zhang2018calm}, continuous data release \cite{LiyueLi-278,Cao2017Quantifying,Theodorakopoulos2014Prolonging,Nath2012MaskIt}, location based services \cite{7542581,Xiao2015Protecting,Liu2017Spatiotemporal}, and social networks \cite{Rastogi2009Relationship,Chen2014Correlated}.
 To achieve efficient privacy protection for correlated data, it is essential to conduct rigorous theoretical studies to understand the analytical relationship between prior knowledge and privacy leakage, which is the main goal of this paper.

There have been several research efforts to this fundamental problem.
The sequential composition theorem~\cite{dwork2014algorithmic} of DP states that correlated data causes linear incrementing of privacy leakage if simply treating the correlated data as a whole.
However, this does not utilize the correlation sufficiently and leads to a low utility for weakly correlated data.
Therefore, many works~\cite{zhang2017privbayes,LiyueLi-278,Chen2017Correlated,Liu2017Spatiotemporal,Liu2016Dependence} have focused on exploiting correlations to achieve high utility without sacrificing the privacy guarantee. However, these works do not consider adversaries with different prior knowledge, which has significant impacts on privacy leakage.
Specifically, it has been demonstrated that without assumption on the adversaries' prior knowledge, no privacy guarantee can be achieved \cite{Kifer2011No,Dwork2010On}.
To measure the impacts of prior knowledge, Pufferfish privacy \cite{Kifer2014Pufferfish} and Blowfish privacy \cite{He2014Blowfish} formally model prior knowledge in their mathematical privacy definitions. However, there are no analytical impacts of correlation and prior knowledge on privacy leakage provided in either work \cite{Kifer2014Pufferfish,He2014Blowfish}.

The state-of-the-art research, Bayesian differential privacy (BDP) \cite{Yang2015Bayesian}, explicitly describes the relationship of privacy leakage and prior knowledge for a special case, i.e., \textit{when data are positively correlated}. However, different types of correlations mean that the maximal influence of the query result caused by one tuple, i.e., the sensitivity, is different, and thus, leading to different privacy leakage. Therefore, it is necessary to discuss privacy leakage under all types of correlations, ranging from $-1$ to $1$ (including negative, independent, and positive correlations).
As BDP is based on a Laplacian matrix that can only model the positive correlations for sum queries, the analytical method and conclusions in \cite{Yang2015Bayesian} cannot be generalized to negative correlations or hybrid correlations (i.e., positive and negative coexist).

In summary, the analytical relationship between prior knowledge and privacy leakage under \textit{general correlations remains unclear}.
To address this problem, this paper presents \textit{the first unified analysis} that considers positive, negative, and hybrid data correlations. Our contributions are as follows:



 \begin{enumerate}
   \item We propose the definition of prior differential privacy (PDP) to measure privacy leakage caused by an adversary with any prior knowledge under general correlations. Based on PDP, we present a unified formulation (Theorem \ref{TH:privacy of two adjacent nodes}) to measure and discuss (Theorem \ref{TH:edge and correlation}) the impact of privacy leakage under varied prior knowledge and data correlations. Both the formulation and discussion can help us better understand the impact of prior knowledge and data correlation on privacy leakage.


   \item We analyze privacy leakage for both discrete and continuous data. For discrete data, we propose a graph model to present the structure of the adversaries' prior knowledge, and a chain rule (Theorem \ref{TH:chain rule}) to compute the privacy leakage. For continuous data, instead of a Markov random field, we adopt the multivariate Gaussian model to present general data correlations and derive a closed-form expression to compute privacy leakage (Theorem \ref{TH:privacy of Gaussian model}).
       Our analytic method is based on the theory of Bayesian inference.
       The analytical results can guide us in designing more efficient mechanisms with better utility-privacy tradeoffs. 

   \item We demonstrate that the analytic results can be applied to general linear queries, including count, sum, mean, and histogram. 
       Extensive numerical simulation results verify our theoretical analysis.

 \end{enumerate}

 The remainder of this paper is organized as follows. Section \ref{Sec:Relatedwork} introduces the related work. Section \ref{Sec:Preli} introduces the notations and presents some preliminary knowledge. In section \ref{Sec:prior differential privacy}, a new definition PDP is proposed to analyze the impacts of prior knowledge, and we illustrate that three factors can impact privacy leakage. Section \ref{Section:WHG for discrete data} and Section \ref{Section:continuous} present the theoretical analysis of privacy leakage for both discrete data and continuous data, respectively.
 Numerical experiments are presented in Section \ref{Sec:Experiments}, and we conclude this paper in Section \ref{Sec:Conclusion}.

\section{Related Work}
\label{Sec:Relatedwork}

\subsection{Data Correlation}

 Many studies \cite{Kifer2011No,Chen2017Correlated,Liu2017Spatiotemporal,Liu2016Dependence} have demonstrated that DP may not guarantee its expected privacy when data are correlated. There are two plausible solutions to protecting the privacy of correlated data records. One is to achieve DP on each data record independently. However, the composition theorem~\cite{dwork2014algorithmic} of DP has demonstrated that the privacy guarantee degrades with the number of correlated records.
 Another is to take the data records as a whole \cite{Dwork2006Differential,Chen2014Correlated,Dwork2011A}. However, when the number of records is large, or the correlation is weak, the utility will still be low.

 Therefore, it is crucial to accurately measure the data correlations to achieve more efficient privacy protection. Considerable work has been done from different perspectives. For general correlations, some work replaces the global sensitivity with new correlation-based parameters, such as correlated sensitivity \cite{Zhu2014Correlated} and correlated degree \cite{Cao2015Coupled}. For example, in \cite{Zhu2014Correlated}, a correlation coefficient matrix was utilized to describe the correlation of a series, and the correlation coefficient was considered as the weight to compute the global sensitivity. By utilizing inter- and intra-coupling, \cite{Cao2015Coupled} proposed behavior functions to model the degree of correlation. For temporal correlations, most of the research work has focused on saving the privacy budget consumption in time-series data \cite{dwork2010differential,LiyueLi-278,Xiao2015Protecting,Nath2012MaskIt}. For example, Dwork \cite{dwork2010differential} proposed a cascade buffer counter algorithm to adaptively update the output result on an $\{0,1\}$ data stream. Fan \cite{LiyueLi-278} adopted a PID controller-based sampling strategy to adaptively inject Laplace noise into time-series data to improve the utility. For spatial correlations, the main idea is to group and perturb the statistics over correlated regions to avoid noise overdose~\cite{7542581, Kellaris2013Practical}. As a typical example, Wang \cite{7542581} proposed dynamically grouping the sparse regions with similar trends and adding the same noise to reduce errors. In addition, for attribute correlations in multiattribute datasets, the fundamental idea is to reduce the dimensionality via identifying the attribute correlations \cite{Zhang2014PrivBayes,Ren2018}. For example, Zhang et al. \cite{Zhang2014PrivBayes} constructed a Bayesian network to model the attribute correlation in high-dimensional data and then synthesized a privacy-preserving dataset in an ad hoc way.
 However, all these works assumed that adversaries have fixed prior knowledge, and thus, may not achieve the optimal tradeoff against adversaries with prior knowledge. In this paper, we consider both data correlations and flexible prior knowledge.


\subsection{Prior Knowledge}

 Prior knowledge can influence privacy leakage when the data are correlated \cite{Chen2017Correlated,Yang2015Bayesian,wu2017extending}, which has been considered in different research in terms of privacy definition and the design of privacy-preserving mechanisms.
 For example, the Pufferfish framework \cite{Kifer2014Pufferfish}, aiming to help domain experts customize privacy definitions, theoretically has the potential to include all kinds of adversaries. The subsequent work of Blowfish privacy \cite{He2014Blowfish} developed mechanisms that permit more utility by specifying secrets about individuals and constraints about the data. In \cite{wang2016privacy}, a Wasserstein mechanism was proposed to fulfill the Pufferfish framework. In addition, \cite{wu2017extending} studied privacy leakage caused by the weakest adversary, and proposed the identity differential privacy (IDP) model. \cite{Chanyaswad:2018:MMD:3243734.3243750} exploited the structural characteristics of databases and the prior knowledge of domain experts to improve utility. However, no theoretical analysis on the relationship between the prior knowledge and privacy leakage has been formulated in all these work. In some research \cite{dimitrakakis2017differential,zhang2016differential}, privacy leakage was guaranteed by limiting the difference between prior knowledge and posterior knowledge. However, in these works, the adversaries' prior knowledge was limited to the probability distribution of the database and did not consider that partial data records may be compromised by specific adversaries. 
 Instead, \cite{Yang2015Bayesian,Chen2017Correlated} separated the adversary's specific prior knowledge of partial tuples from the public knowledge of data correlations, which are derived from data distributions. 
 Based on that, Yang et al. \cite{Yang2015Bayesian} adopted a Gaussian correlation model to study the impact of prior knowledge and demonstrated that the weakest adversary could cause the highest privacy leakage. Similar conclusions can be found in \cite{Chen2017Correlated}, which further identifies the maximally correlated group of data tuples to improve the utility.
 Nonetheless, the limitation is that their Laplacian matrix based Markov random field model can only be applied to analyze positive correlations on sum queries for continuous data or binary discrete data.

 In this paper, we formally derive a formulation to present a unified analysis of the impact of data correlation and prior knowledge on privacy leakage, considering general linear queries on both discrete and continuous data.

\section{Preliminaries}
\label{Sec:Preli}
 We describe notations and conceptions in Subsection \ref{Subsection:preliminary-notations}, and introduce some knowledge of DP that will be used in our analysis in Subsection \ref{Subsection:preliminary-DP}.
 \subsection{Notations}
 \label{Subsection:preliminary-notations}

 A database with $n$ tuples (attributes in a table or nodes in a graph), denoted as the set of indices $[n]=\{1,2,\cdots,n\}$, aims to release the result of a certain query function $s=f(\mathbf{x})$ on an instance of the database, $\mathbf{x}=\{x_1,x_2,\cdots,x_n\}$. It should be noted that, in accordance with \cite{Yang2015Bayesian,Liu2016Dependence,Chen2017Correlated}, we use the same term ``tuple'' to denote the attribute instead of the record in a database. To protect the privacy of all tuples of an instance, it will return the noisy answer $r=\mathcal{M}(f(\mathbf{x}))$ by adding random noise drawn from a distribution. Hence, all possible outputs $S$ constitute a probability distribution Pr$(\mathcal{M}(f(\mathbf{x}))\in S)$, or equivalently a conditional distribution Pr$(r\in S|f(\mathbf{x})=s)$. We use a set $\Theta$ to capture the adversary's beliefs on data correlation. We do not guarantee the privacy against adversaries out of $\Theta$,
 because there is no feasibility under arbitrary distributions \cite{Li2013Membership}. The main notations are listed in Table \ref{Tab:notations}.


 \begin{table}[!t]
 \caption{Notations and meanings}
 \label{Tab:notations}
 \begin{tabular}{|c|c|}
  \hline
  notations  & descriptions \\\hline\hline
  $\mathbf{x}$ & A database instance $\{x_1,x_2,\ldots,x_n\}$. \\\hline
  $\mathcal{U,K}$ & The indices set of unknown/known tuples. \\\hline
  $\mathbf{x}_{\mathcal{U}},\mathbf{x}_{\mathcal{K}}$ & The instances of unknown/known tuples. \\\hline
  $s,s_{\mathcal{K}},s_{\mathcal{U}}$ & The sum of instance $\mathbf{x}$, $\mathbf{x}_{\mathcal{K}}$ and $\mathbf{x}_{\mathcal{U}}$. \\\hline
  $x_i,x_i'$ & Two different values of tuple $i$.  \\\hline
  $\mathbf{x}_{-i}$ & The database $\mathbf{x}$ with $x_i$ eliminated. \\\hline
  $\mathbf{x}'$ & The database $\mathbf{x}$ with $x_i$ replaced with $x_i'$. \\\hline
  $\mathcal{A}_{i,\mathcal{K}}$ & An adversary with prior knowledge $\mathbf{x}_{\mathcal{K}}$ to attack $x_i$. \\\hline
  $l_{\mathcal{A}_{i,\mathcal{K}}}$ & The privacy leakage caused by the adversary $\mathcal{A}_{i,\mathcal{K}}$. \\\hline
  $r\in \mathcal{R}$ & The random request generated by $\mathcal{M}$. \\\hline
  $\mathcal{M}$ & A randomized mechanism over $\mathbf{x}$. \\\hline
  $\theta\in\mathbf{\Theta}$ & All possible distributions of $\mathbf{x}$. \\\hline
  $LS_i(f)$ & The local sensitivity of a query function on tuple $i$. \\\hline
  $GS(f)$ & The global sensitivity of a query function on $\mathbf{x}$. \\\hline
 \end{tabular}
 \end{table}

\subsubsection{Adversary and Prior Knowledge}

 We denote $\mathcal{A}_{i,\mathcal{K}}$ as an adversary who attempts to infer the information of tuple $x_i$, under the assumption that he knows the values of $\mathbf{x}_{\mathcal{K}}$. We call $x_i$ the \textbf{attack object}, $\mathbf{x}_{\mathcal{K}}$ is the \textbf{prior knowledge}, $\mathcal{K} \subseteq [n]\setminus\{i\}$, where $[n]=\{1,2,\cdots,n\}$. Let $\mathcal{U}$ denotes the indices set of unknown tuples, then $[n]=\mathcal{K}\cup\{i\}\cup\mathcal{U}$ and the dataset $\mathbf{x}=\{\mathbf{x}_{\mathcal{K}}, x_i, \mathbf{x}_{\mathcal{U}}\}$.
 An adversary $\mathcal{A}_{i,\mathcal{K}}$ is called the strongest adversary when $\mathcal{K}=[n]\setminus\{i\}$ and is called the weakest adversary when $\mathcal{K}=\emptyset$. $\mathcal{A}_{i,\mathcal{K}'}$ is called an ancestor of $\mathcal{A}_{i,\mathcal{K}}$ if $\mathcal{K}'$ is a subset of $\mathcal{K}$ and differs by only one tuple, i.e., $\mathcal{K}'=\mathcal{K}\backslash\{j\}$. More tuples in $\mathbf{x}_{\mathcal{K}}$ mean the adversary has stronger prior knowledge.

 \subsubsection{Correlation}

 To measure data correlations, we adopt the Pearson correlation coefficient, which can identify linear correlations. More importantly, it can be used to distinguish positive correlations and negative correlations. In joint distribution $\theta$, let $\rho_{ij,\mathcal{K}}$ denote the correlation coefficient of $x_i$ and $x_j$ under the condition $\mathbf{x}_{\mathcal{K}}$. In this paper, $\rho_{ij,\mathcal{K}}$ plays an important role in the analysis of how prior knowledge affects privacy leakage.

 \subsubsection{Linear Query}

 A linear query function can be represented as $f(\mathbf{x})=\sum_i a_i x_i$, where $x_i, x_j \in \mathbf{x}$ are correlated with the Pearson correlation coefficient $\rho_{ij}$. The linear query function can be transformed into a sum query $f(\mathbf{y})=\sum_{i} y_i$ on a new database $\mathbf{y}$ by letting $a_i x_i$ as $y_i$, where $y_i \in \mathbf{y}$. Then, the correlation coefficient of $y_i$ and $y_j$ should be $\rho_{ij}'=sign(a_i a_j) \rho_{ij}$. Combining our new privacy definition PDP (will be discussed in Subsection \ref{Subsection:PDP-definition}), models can deal with general correlations; therefore, we focus our analysis on the sum query without loss of generality, and the conclusion can be straightforwardly extended to general linear queries.

 \subsection{Differential Privacy}
 \label{Subsection:preliminary-DP}

 \begin{defe}\label{Def:DP}
 (Differential Privacy \cite{Dwork2006Calibrating}).
   \emph{A randomized mechanism $\mathcal{M}$ satisfies $\varepsilon$-differential privacy ($\varepsilon$-DP), if for any $S\subseteq Range\{\mathcal{M}\}$, the differential value $x_i,x_i'$
 \begin{equation}\label{Eq:DP}
   DP(\mathcal{M})=\sup_{i,\mathbf{x}_{-i},x_i,x_i',S,}\log\frac{\Pr(r\in S|x_i,\mathbf{x}_{-i})}{\Pr(r\in S|x_i',\mathbf{x}_{-i})}\leq \varepsilon.
 \end{equation}}
 \end{defe}

 Here, $\varepsilon>0$ is the distinguishable bound of all outputs on neighboring datasets $\mathbf{x}$ and $\mathbf{x}'$, where $\mathbf{x}'$ is the database $\mathbf{x}$ with $x_i$ replaced with $x_i'$. A larger $\varepsilon$ corresponds to easier distinguishability of $x_i$ and $x_i'$, which means more privacy leakage.

 For numerical data, a Laplace mechanism \cite{Dwork2006Calibrating} can be used to achieve $\varepsilon$-DP, by adding carefully calibrated noise to the query results. In particular, we draw noise from Laplace distribution $Lap(\lambda)$ with the probability density function
 $$p(z)=\frac{1}{2\lambda}\mathrm{exp}(-{|z|}/{\lambda}),$$
 in which $\lambda={GS(f)}/{\varepsilon}$. Here, $GS(f)=\sup_{\mathbf{x},\mathbf{x}'}\|f(\mathbf{x})-f(\mathbf{x}')\|_1$ is the global sensitivity of query $f(\cdot)$, and $LS_i(f)=\sup_{\mathbf{x}'}\|f(\mathbf{x})-f(\mathbf{x}')\|_1$ is the local sensitivity of $f$. Since $r=f(\mathbf{x})+z$, the probability density function of the output can be represented as
 $$p(r|\mathbf{x})=\frac{1}{2\lambda}\mathrm{exp}({-|r-f(\mathbf{x})|/{\lambda}}).$$

 \section{Prior Differential Privacy}
 \label{Sec:prior differential privacy}
 To compute the privacy leakage when considering adversaries with different prior knowledge and databases with different joint distributions, we propose a new definition in Subsection \ref{Subsection:PDP-definition}. Furthermore, we illustrate that three factors can affect privacy leakage through three numerical examples in Subsection \ref{Subsection:PDP-Influence factors}.

 \subsection{Prior Differential Privacy}
 \label{Subsection:PDP-definition}

 To evaluate privacy leakage considering adversaries have different prior knowledge, the definition BDP is proposed in \cite{Yang2015Bayesian} based on the Bayesian inference method \cite{Li2013Membership,Lee:2012:DI:2339530.2339695,Kifer:2012:RCF:2213556.2213571}. However, BDP can only be applied to positive correlations. To overcome the drawback, we propose a definition named \textbf{Prior Differential Privacy (PDP)}, which can be applied to databases with general correlations.

 \begin{defe}\label{Def:PDP}
   (Prior Differential Privacy) \emph{Let $\mathbf{x}$ be a database instance with $n$ tuples, $\mathcal{A}_{i,\mathcal{K}}$ is an adversary with the attack object $x_i$ and prior knowledge $\mathbf{x}_\mathcal{K}, \mathcal{K}\subseteq [n]\setminus\{i\}$. The joint distribution of $\mathbf{x}$ is denoted as $\theta, \theta\in\Theta$, where $\Theta$ is a set of distributions. $\mathcal{M}=\mathrm{\Pr}(r \in S|\mathbf{x})$ is a randomized perturbation mechanism, and $S$ is the output space. The privacy leakage of $\mathcal{M}$ w.r.t $\mathcal{A}_{i,\mathcal{K}}$ is the maximum logarithm function for all different values $x_i$, $x_i'$, and any output $r\in S$.
 \begin{equation}\label{Eq:definition of PDP}
    l_{\mathcal{A}_{i,\mathcal{K}}}(\theta)=\sup_{x_i,x_i',r}\log\frac{\Pr(r\in S|x_i,\mathbf{x}_{\mathcal{K}})}{\Pr(r\in S|x_i',\mathbf{x}_{\mathcal{K}})}.
  \end{equation}
 We say $\mathcal{M}$ satisfies $\varepsilon$-PDP if Eq.\;(\ref{Eq:definition of PDP}) holds for any $i\in [n]$, $\mathcal{K}\subseteq [n]\backslash\{i\}$, $\theta\in\Theta$. That is,
 \begin{align*}
   \sup_{i,\mathcal{K},\theta} l_{\mathcal{A}_{i,\mathcal{K}}}(\theta)\leq \varepsilon.
 \end{align*}   }
 \end{defe}

 In Definition \ref{Def:PDP}, $l_{\mathcal{A}_{i,\mathcal{K}}}(\theta,\mathcal{M})$ is the privacy leakage caused by the adversary $\mathcal{A}_{i,\mathcal{K}}$ under the distribution $\theta$, which represents the data correlation. $\varepsilon$ is the maximal privacy leakage caused by all adversaries with public distribution $\Theta$. Compared with BDP that only considers a single distribution, PDP considers a set of distributions $\Theta$. Thus, PDP is more reasonable because the set $\Theta$ can reflect the cognitive diversity of the aggregator and the adversaries.

 We show that PDP is in accordance with the Bayesian inference, Eq.\;(\ref{Eq:definition of PDP}) can be written as
 \begin{align}\label{Eq:respective to the Bayesian inference}
   \sup_{x_i,x_i',r}\left(\log\frac{\Pr(x_i|r,\mathbf{x}_{\mathcal{K}})}{\Pr(x_i'|r,\mathbf{x}_{\mathcal{K}})}
   -\log\frac{\Pr(x_i|\mathbf{x}_{\mathcal{K}})}{\Pr(x_i'|\mathbf{x}_{\mathcal{K}})}\right).
 \end{align}
 Eq.\;(\ref{Eq:respective to the Bayesian inference}) denotes the information gain achieved by the adversary $\mathcal{A}_{i,\mathcal{K}}$, after the adversary observes the published results $r$. In addition, the PDP bounds the maximal information gain inferred by all possible adversaries that are no larger than $\varepsilon$. The next theorem shows that prior knowledge impacts privacy leakage only when the database is correlated.

 \begin{thm}\label{TH:prior knowledge has no impact}
   Prior knowledge has no impact on privacy leakage when tuples in the database are mutually independent.
 \end{thm}

 \begin{proof}
   For an adversary $\mathcal{A}_{i,\mathcal{K}}$ and its ancestor $\mathcal{A}_{i,\mathcal{K}'}, \mathcal{K}'=\mathcal{K}\backslash \{j\}$, we get
   \begin{align*}
     \Pr(r\in S|x_i,\mathbf{x}_{\mathcal{K}'})&=\sum_{x_j}\Pr(x_j|x_i,\mathbf{x}_{\mathcal{K}'})\Pr(r\in S|x_i,\mathbf{x}_{\mathcal{K}})\\
     &=\sum_{x_j}\Pr(x_j)\Pr(r\in S|x_i,\mathbf{x}_{\mathcal{K}}).
   \end{align*}
   The last equality holds when the data tuples are independent, i.e., $\Pr(x_j|x_i,\mathbf{x}_{\mathcal{K}'})=\Pr(x_j)$.
   And similarly,
   $ \Pr(r\in S|x_i',\mathbf{x}_{\mathcal{K}'})=\sum_{x_j}\Pr(x_j)\Pr(r\in S|x_i',\mathbf{x}_{\mathcal{K}}).$
   If $l_{\mathcal{A}_{i,\mathcal{K}}}(\theta,\mathcal{M})\leq\varepsilon$, according to PDP, we have $\frac{\Pr(r\in S|x_i,\mathbf{x}_{\mathcal{K}})}{\Pr(r\in S|x_i',\mathbf{x}_{\mathcal{K}})}\in [e^{-\varepsilon},e^{\varepsilon}]$. Multiplying this fraction by $\Pr(x_j)$ and summing with respect to $x_j$, we obtain $l_{\mathcal{A}_{i,\mathcal{K}'}}(\theta,\mathcal{M})\leq \varepsilon$ on the basis of the definition PDP. Therefore, different prior knowledge $\mathcal{K}$ and $\mathcal{K}'$ have the same privacy leakage.
 \end{proof}

 Theorem \ref{TH:prior knowledge has no impact} is also consistent with Eq. (\ref{Eq:respective to the Bayesian inference}). If tuples are independent, then $\mathbf{x}_{\mathcal{K}}$ can be omitted in Eq.\;(\ref{Eq:respective to the Bayesian inference}). Therefore, the prior knowledge has no impact on the privacy leakage when tuples are independent.

 \noindent\textbf{Remark 1.} It is worth noting that DP and PDP are also consistent in nature. They all reflect the maximal distinguishability between distributions of perturbed output calculated on two neighboring datasets. In this paper, neighboring datasets are obtained by modifying one record in the dataset. The difference in DP and PDP is the different forms of neighboring datasets. In DP, the neighboring datasets are $\{x_i,\mathbf{x_{-i}}\}$ and $\{x_i',\mathbf{x_{-i}}\}$. However, in PDP, the neighboring datasets are $\{x_i,\mathbf{x}_{\mathcal{K}}\}$ and $\{x_i',\mathbf{x}_{\mathcal{K}}\}$. For any given $\theta$ and $\mathbf{x}_{\mathcal{K}}$, we have
 \begin{equation}\label{Eq:relationship between DP and PDP}
 \begin{split}
   \frac{\Pr(r\in S|x_i,\mathbf{x}_{\mathcal{K}})}{\Pr(r\in S|x_i',\mathbf{x}_{\mathcal{K}})}
  &=\frac{\sum_{\mathbf{x}_{\mathcal{U}}}\Pr(\mathbf{x}_{\mathcal{U}}|x_i,\mathbf{x}_{\mathcal{K}})\Pr(r|x_i,\mathbf{x}_{-i})}{\sum_{\mathbf{x}_{\mathcal{U}}}
  \Pr(\mathbf{x}_{\mathcal{U}}|x_i',\mathbf{x}_{\mathcal{K}})\Pr(r|x_i',\mathbf{x}_{-i})}\\
  &\leq \sup_{\mathbf{x}_{\mathcal{U}},\mathbf{x}_{\mathcal{U}}'}\frac{\Pr(r|x_i,\mathbf{x}_{\mathcal{U}},\mathbf{x}_{\mathcal{K}})}
  {\Pr(r|x_i',\mathbf{x}_{\mathcal{U}}',\mathbf{x}_{\mathcal{K}})}\\
  &=\sup_{\mathbf{x}_{i\cup\mathcal{U}},\mathbf{x}_{i\cup\mathcal{U}}'}\frac{\Pr(r|\mathbf{x}_{i\cup\mathcal{U}},\mathbf{x}_{\mathcal{K}})}
  {\Pr(r|\mathbf{x}_{i\cup\mathcal{U}}',\mathbf{x}_{\mathcal{K}})}.
 \end{split}
 \end{equation}
 The last equality in Eq.\;(\ref{Eq:relationship between DP and PDP}) applies DP on datasets differing at most $|\mathcal{U}|$+1 tuples. The inequality in Eq.\;(\ref{Eq:relationship between DP and PDP}) holds for the fact that
 $\frac{\sum_{i}a_i c_i}{\sum_{i}b_i d_i}\leq \max_{i,j} \frac{c_i}{d_j}$,
 if all parameters are nonnegative and $\{a_i\}$ and $\{b_i\}$ are probability simplex. Taking the logarithm and supremum of Eq.\;(\ref{Eq:relationship between DP and PDP}) over $r$, we obtain $l_{\mathcal{A}_{i,\mathcal{K}}}(\theta,\mathcal{M})\leq DP(\mathcal{M})$.
 Therefore, DP provides an upper bound of privacy leakage for PDP, i.e., we achieve a better trade-off between privacy and utility by adopting PDP than adopting DP.

\subsection{Influence Factors}
\label{Subsection:PDP-Influence factors}

 In this section, we demonstrate that three factors, prior knowledge $\mathbf{x}_{\mathcal{K}}$, joint distribution $\theta$, and local sensitivity $LS_i(f)$, impact privacy leakage through numerical Example 2 to Example 4.

 \begin{table}[htpb]
\center
\caption{Four Joint Distributions}\label{Tab:four distributions}
\subtable[Positive correlation]{
       \begin{tabular}{|c|c|c|}
       \hline
          & $x_1$=0 & $x_1$=1\\\hline
          $x_2$=0 & 0.3 & 0.2\\\hline
          $x_2$=1 & 0.2 & 0.3\\\hline
       \end{tabular}\label{Tab:four distributions-a}
}
\subtable[Negative correlation]{
       \begin{tabular}{|c|c|c|}
       \hline
          & $x_1$=0 & $x_1$=1\\\hline
          $x_2$=0 & 0.2 & 0.3\\\hline
          $x_2$=1 & 0.3 & 0.2\\\hline
       \end{tabular}\label{Tab:four distributions-b}
}
\subtable[Perfect correlation]{
       \begin{tabular}{|c|c|c|}
       \hline
          & $x_1$=0 & $x_1$=1\\\hline
          $x_2$=0 & 0.5 & 0\\\hline
          $x_2$=1 & 0 & 0.5\\\hline
       \end{tabular}\label{Tab:four distributions-c}
}
\subtable[Perfect correlation]{
       \begin{tabular}{|c|c|c|}
       \hline
          & $x_1$=0 & $x_1$=1\\\hline
          $x_2$=0 & 0.5 & 0\\\hline
          $x_2$=5 & 0 & 0.5\\\hline
       \end{tabular}\label{Tab:four distributions-d}
}
\end{table}

 As shown in Table\;\ref{Tab:four distributions}, there are four joint distributions of database $\mathbf{x}=\{x_1,x_2\}$. The first three distributions have the same domain $x_1,x_2\in\{0,1\}$ with different correlations, the third and fourth distributions have the same correlation, but $x_2$ has a different domain. Considering a sum query $f(\mathbf{x})=x_1+x_2$, set Laplace mechanism scale $\lambda=1$ for simplicity. Denote $l_{\mathcal{A}_{1,\mathcal{K}}}(a)$ as the privacy leakage caused by the adversary $\mathcal{A}_{1,\mathcal{K}}$ when the distribution of $\mathbf{x}$ is Table\;\ref{Tab:four distributions-a}.

 Example 2 (Prior Knowledge). \emph{Two adversaries $\mathcal{A}_{1,\{2\}}$ and $\mathcal{A}_{1,\emptyset}$, attempt to infer the information $x_1=0$ or $x_1=1$. $\mathcal{A}_{1,\{2\}}$ knows the information of $x_2$ (e.g., $x_2=1$), and $\mathcal{A}_{1,\emptyset}$ knows nothing about $x_2$. Based on the definition of PDP, we calculate
 $l_{\mathcal{A}_{1,\{2\}}}(a)$ and $l_{\mathcal{A}_{1,\emptyset}}(a)$. For $\mathcal{A}_{1,\{2\}}$ and $x_2=1$, we get
 \begin{align*}
    l_{\mathcal{A}_{1,\{2\}}}(a)&=\sup_{r}\log\frac{\Pr(r|x_1=0,x_2=1)}{\Pr(r|x_1=1,x_2=1)}\\
    &=\sup_{r}\log\frac{\exp(-|r-1|)}{\exp(-|r-2|)}=1.
 \end{align*}
 When $\mathcal{A}_{1,\emptyset}$ knows nothing about $x_2$, according to Eq.\;(\ref{Eq:definition of PDP}), we have
 \begin{align*}
   l_{\mathcal{A}_{1,\emptyset}}(a)&=\sup_{r}\log\frac{\Pr(r|x_1=0)}{\Pr(r|x_1=1)}\\
   &=\sup_{r}\log\frac{\sum_{x_2}\Pr(x_2|x_1=0)\exp(-|r-(0+x_2)|)}{\sum_{x_2}\Pr(x_2|x_1=1)\exp(-|r-(1+x_2)|}\\
   &\approx 1.19.
 \end{align*}
 The exponential entries are derived from the Laplace mechanism and given $x_1,x_2$. Similarly, $l_{\mathcal{A}_{1,\emptyset}}(b)\approx 0.82.$
 Therefore,
  \begin{align}
     l_{\mathcal{A}_{1,\emptyset}}(a)>l_{\mathcal{A}_{1,\{2\}}}(a),\label{Eq:prior impact under positve correlation}\\
     l_{\mathcal{A}_{1,\emptyset}}(b)<l_{\mathcal{A}_{1,\{2\}}}(b)\label{Eq:prior impact under negative correlation}.
  \end{align}}

 Example 2 shows that prior knowledge has significant influence when the correlations are different. More importantly, it answers the two problems extended from Example 1. In addition, we note that the privacy leakage of DP is 2 if we simply regard correlated tuples $x_1$, $x_2$ as a whole. Therefore, we achieve stricter privacy protection than DP under the same noise mechanism. In other words, we can introduce less noise to obtain the same privacy level.

 Example 3 (Correlation). \emph{An adversary $\mathcal{A}_{1,\emptyset}$ attempts to infer the information of $x_1$ with no prior knowledge about $x_2$. To show the impacts of the correlations, we modify $0.3\rightarrow 0.49, 0.2\rightarrow 0.01$ in Tables\;\ref{Tab:four distributions-a} and \ref{Tab:four distributions-b} to obtain two distributions (a') and (b'), in which $x_1$ and $x_2$ have stronger correlation. Computations of $l_{\mathcal{A}_{1,\emptyset}}(a')$ and $l_{\mathcal{A}_{1,\emptyset}}(b')$ are similar to $l_{\mathcal{A}_{1,\emptyset}}(a)$ in Example 2. According to Eq.\;(\ref{Eq:definition of PDP}), we obtain $l_{\mathcal{A}_{1,\{2\}}}(a')=1.95$, $l_{\mathcal{A}_{1,\{2\}}}(b')=0.05$. Therefore,
  \begin{align}
     l_{\mathcal{A}_{1,\emptyset}}(a')> l_{\mathcal{A}_{1,\emptyset}}(a),\label{Eq:positve correlation impact}\\
     l_{\mathcal{A}_{1,\emptyset}}(b')< l_{\mathcal{A}_{1,\emptyset}}(b)\label{Eq:negative correlation impact}.
  \end{align}}

 Example 3 demonstrates that different correlations have significant influences on privacy leakage. Particularly, Eq.\;(\ref{Eq:positve correlation impact}) shows that the adversary can infer more information of $x_1$ through a stronger positive correlation, and Eq.\;(\ref{Eq:negative correlation impact}) shows the opposite result when the correlation is negative.

 Example 4 (Local Sensitivity). \emph{An adversary $\mathcal{A}_{1,\emptyset}$, with no prior knowledge of $x_2$, attempts to infer $x_1$. The difference in distributions Tables\;\ref{Tab:four distributions-c} and \ref{Tab:four distributions-d} is the domain of $x_2$. Based on PDP and the similarity of computations of  $l_{\mathcal{A}_{1,\emptyset}}(a)$ in Example 2, we have $l_{\mathcal{A}_{1,\emptyset}}(c)=2$, and $l_{\mathcal{A}_{1,\emptyset}}(d)=6$. Therefore,
 \begin{align*}
   l_{\mathcal{A}_{1,\emptyset}}(c) < l_{\mathcal{A}_{1,\emptyset}}(d).
 \end{align*}}
 For the sum query on $\mathbf{x}$, the local sensitivity of $x_i$ is its own domain. In distribution Table\;\ref{Tab:four distributions-c}, $LS_2(f)/LS_1(f)=1$. In distribution Table\;\ref{Tab:four distributions-d}, $LS_2(f)/LS_1(f)=5$. Example 4 shows that the local sensitivity impacts privacy leakage and a larger sensitivity ratio can lead to higher privacy leakage.

 Examples 2-4 demonstrate that three factors impact privacy leakage, and show how to compute the privacy leakage for a database composed of two tuples. In the following sections, we will extend the numerical results to analytical results for both discrete and continuous data.

\section{General Relationship Analysis and Privacy Leakage Computation for Discrete Data}
\label{Section:WHG for discrete data}

 In this section, we analyze privacy leakage with respect to the three factors when data are discrete. Subsection \ref{Subsection:discrete-WHG} presents a weighted hierarchical graph (WHG) to model all adversaries with various prior knowledge. Subsection \ref{Subsection:discrete-edgevalue} discusses how to calculate the weight of edges in the WHG. Subsection \ref{Subsection:discrete-chainrule} formulates a chain rule to represent the privacy leakage for an adversary with arbitrary prior knowledge. Subsection \ref{Subsection:discrete-algorithm} presents a full-space-searching algorithm to compute the privacy leakage, and a fast-searching algorithm to improve the search efficiency in practice.

 \subsection{Weighted Hierarchical Graph}
 \label{Subsection:discrete-WHG}

 A hierarchical graph is used to represent adversaries with various prior knowledge. Each node $(i,\mathcal{K})$ denotes an adversary, in which tuple $i$ is the attack object, and tuples set $\mathcal{K}$ denotes the prior knowledge. For a database with $n$ tuples, there are $n$ layers in a graph. From the bottom to the top, the prior knowledge $\mathcal{K}$ decreases by one tuple for each layer, until $\mathcal{K}=\emptyset$.
 To compute the privacy leakage of adversaries, we further construct a weighted hierarchical graph (WHG) by assigning weights for the edges in the graph. We first define the value of nodes as the privacy leakage caused by corresponding adversaries. In addition, the edge connecting two nodes denotes the privacy leakage difference between two adversaries with the neighboring prior knowledge sets, i.e., $|\mathcal{K}|-|\mathcal{K}'|=1$. Then, the process of analyzing the privacy leakage is as follows. First, we construct the hierarchical graph for all possible adversaries for a given database. Second, we compute the values of all edges in the graph to obtain the WHG (discussed in Subsection \ref{Subsection:discrete-edgevalue}). Third, we compute the values of nodes in the first layer by PDP. Finally, we can obtain all nodes' values by proposing a chain rule (Theorem \ref{TH:chain rule}). Finally, the privacy leakage can be obtained by choosing the maximal node naturally.



 For example, we can obtain a WHG consisting of three layers and twelve nodes for a simple database with three tuples, as shown in Fig~\ref{Fig:WHGof three tuples}. Based on the node $(2,\{1,3\})$ and edges $e_{23,1}$ and $e_{21,3}$, we obtain the privacy leakage for nodes $(2,\{1\})$ and $(2,\{3\})$. Similarly, we can obtain the other four nodes in the second layer. For the node $(2,\emptyset)$, we compute two values based on $(2,\{1\}),e_{21,\emptyset}$ and $(2,\{3\}),e_{23,\emptyset}$, and choose the minimum as its privacy leakage. Similarly, we obtain another two nodes in the third layer. Now, the privacy leakage is the maximal node value in the graph.

 In the above process, one key problem is to compute the edge value. Therefore, we propose a formula to address the problem.

\begin{figure}
  \centering
  \includegraphics[width=8cm]{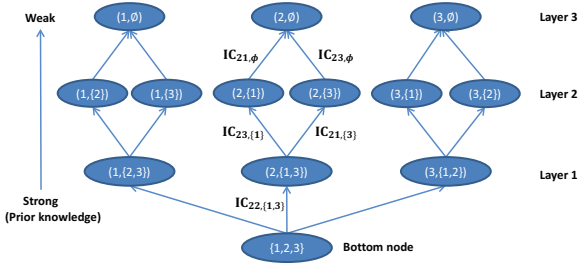}\\
  \caption{An example to show the three tuples WHG. Each node $(i,\mathcal{K})$ denotes an adversary who attempts to infer the tuple $i$ with prior knowledge $\mathcal{K}$. There are three levels composed of nodes with the same prior knowledge size. A directed edge connects the node and its ancestor from the lower layer to higher layer. Therefore, we obtain a directed graph to present all possible adversaries. 
  }\label{Fig:WHGof three tuples}
\end{figure}

 \subsection{Impacts of Correlations and Prior Knowledge}
 \label{Subsection:discrete-edgevalue}

 In this section, we mainly deduce the formula to compute the edge value, which represents the impact of privacy leakage caused by different prior knowledge. Meanwhile, we show that the edge value is closely related to data correlation.

 Note that the edge value shows the gain of privacy leakage when one tuple is removed from the prior knowledge. If the edge value is positive, then the ancestor, a weaker adversary, can cause more privacy leakage. If the edge value is negative, then the ancestor, a stronger adversary, can cause more privacy leakage.


 Given $\mathbf{x}_{\mathcal{K}}'$, $\Pr(x_i,x_j|\mathcal{K}')$ denotes the conditional distribution derived from the joint distribution $\theta$, and $\rho_{ij,\mathcal{K}'}$ is the corresponding conditional correlation coefficient. The domain of tuple $x_i$ is $\{x_{i,1},x_{i,2},\cdots,x_{i,s}\}$, in which $s$ is the domain size of $x_i$. Based on $\Pr(x_i,x_j|\mathcal{K}')$, the \underline{i}mpa\underline{c}t of $x_i$ on $x_j$, under two different values $x_{i,1},x_{i,2}$ of $x_i$, can be denoted as
 \begin{align}\label{Eq:IC}
   IC_{j,\mathcal{K}'}(x_{i,1},x_{i,2})=\log\frac{\sum_{x_j}\mathrm{Pr}(x_j|x_{i,1},\mathbf{x}_{\mathcal{K}'})e^{{-x_j}/{\lambda}}}
     {\sum_{x_j}\mathrm{Pr}(x_j|x_{i,2},\mathbf{x}_{\mathcal{K}'})e^{{-x_j}/{\lambda}}}.
 \end{align}
 Then, impacts of $x_i$ on $x_j$, under all possible pairs $x_{i,m},x_{i,n}$, can be denoted as a set
 \begin{align}\label{Eq:set of Gamma}
   \Gamma_{ij,\mathcal{K}'}=\{IC_{j,\mathcal{K}'}(x_{i,m},x_{i,n})|\forall x_{i,m},x_{i,n}\in dom{x_i},m<n\}.
 \end{align}

 Next, a theorem shows how to compute the edge value.

 \begin{thm}\label{TH:privacy of two adjacent nodes}
  \emph{Assume the privacy leakage of an adversary $\mathcal{A}_{i,\mathcal{K}}$ is $l_{\mathcal{A}_{i,\mathcal{K}}}$, then the privacy leakage of its ancestor $\mathcal{A}_{i,\mathcal{K}'}(\mathcal{K}'=\mathcal{K}\backslash\{j\})$ is
   \begin{align}\label{Eq:privacy of two adjacent nodes}
      l_{\mathcal{A}_{i,\mathcal{K}'}}=|l_{\mathcal{A}_{i,\mathcal{K}}}+IC_{ij,\mathcal{K}'}|,
   \end{align}
   where
   \begin{align*}
     IC_{ij,\mathcal{K}'}=\mathop{\arg\max}_{\gamma\in\Gamma_{ij,\mathcal{K}'}}|l_{\mathcal{A}_{i,\mathcal{K}}}+\gamma|
   \end{align*}
  is the value of the edge connecting two nodes $(i,\mathcal{K})$ and $(i,\mathcal{K}')$ in the WHG.}
 \end{thm}

 \begin{proof}
   See Appendix \ref{Proof:privacy of two nodes}.
 \end{proof}

 Theorem \ref{TH:privacy of two adjacent nodes} shows the impact on privacy leakage $IC_{ij, \mathcal{K}'}$ caused by two adversaries whose prior knowledge differs by one tuple under general correlation. According to Theorem \ref{TH:privacy of two adjacent nodes}, the value $IC_{ij,\mathcal{K}'}$ is the element in the set $\Gamma_{ij,\mathcal{K}'}$ that maximizes the privacy leakage of node $(i,\mathcal{K}')$. That is, $IC_{ij,\mathcal{K}'}$ presents the maximal impact of $x_i$ on $x_j$ under the conditional distribution $\Pr(x_i,x_j|\mathbf{x}_{\mathcal{K}})$.

 To show the relationship between $IC_{j,\mathcal{K}'}(x_{i,m},x_{i,n})$ and the three factors described in Subsection \ref{Subsection:PDP-Influence factors}, we rewrite the $IC_{j,\mathcal{K}'}(x_{i,m},x_{i,n})$ as
 \begin{align}\label{Eq:IR}
  IC_{j,\mathcal{K}'}(x_{i,m},x_{i,n})=IR_{j,\mathcal{K}'}(x_{i,m},x_{i,n})\cdot\frac{LS_j(f)}{\lambda},
 \end{align}
 where
 \begin{align}\label{Eq:definition of IR}
   IR_{j,\mathcal{K}'}(x_{i,m},x_{i,n})=IC_{j,\mathcal{K}'}(x_{i,m},x_{i,n})/(LS_j(f)/\lambda)
 \end{align}
 is called the \underline{i}ncrement \underline{r}atio to denote the impact caused by correlations. $IC_{j,\mathcal{K}'}(x_{i,m},x_{i,n})$ represents the variation of privacy leakage when prior knowledge decreases. Therefore, the two components $IR_{j,\mathcal{K}'}(x_{i,m},x_{i,n})$ and ${LS_j(f)}/{\lambda}$ in Eq.\;(\ref{Eq:IR}) represent the impact of local sensitivity, and correlation, respectively.


 Now, we give the relationship between $IR_{j,\mathcal{K}'}(x_{i,m},x_{i,n})$ and conditional correlation coefficient $\rho_{ij,\mathcal{K}'}$.
 \begin{thm}\label{TH:edge and correlation}
   \emph{(1) For a database $\mathbf{x}$ with all possible joint distributions, $IR_{j,\mathcal{K}'}(x_{i,m},x_{i,n})\in[-1,1]$. (2) Under the assumption that the domain size of $x_i$ and $x_j$ are two, then $IR_{j,\mathcal{K}'}(x_{i,1},x_{i,2})$ has the following relationship with $\rho_{ij,\mathcal{K}'}$:
   \begin{enumerate}
     \item if $\rho_{ij,\mathcal{K}'}>0$, then $IR_{j,\mathcal{K}'}(x_{i,1},x_{i,2})\in(0,1]$;
     \item if $\rho_{ij,\mathcal{K}'}=0$, then $IR_{j,\mathcal{K}'}(x_{i,1},x_{i,2})=0$;
     \item if $\rho_{ij,\mathcal{K}'}<0$, then $IR_{j,\mathcal{K}'}(x_{i,1},x_{i,2})\in[-1,0)$.
   \end{enumerate}
   (3) Under the assumption that the domain size of $x_i$ is two, the domain size of $x_j$ is greater than two, meanwhile, $\lambda>GS(f)$, and then the results in Case (2) still applies.
   }
 \end{thm}
 \begin{proof}
   See Appendix \ref{Proof:LGR and correlation}.
 \end{proof}
 Case (1) in Theorem \ref{TH:edge and correlation} shows that $IR_{j,\mathcal{K}'}(x_{i,m},x_{i,n})$ has the same bound as the correlation coefficient.
 Case (2) shows that the relationship between the edge value and the data correlations, and extends the results of Examples 2-4 to general cases. Case (3) shows that similar results hold for a general $x_j$ with a larger domain, as long as $\lambda>GS(f)$. Since the privacy budget $\varepsilon=GS(f)/\lambda$ in DP is commonly set as $\varepsilon<1$, the condition $\lambda>GS(f)$ is usually true.
 Theorem \ref{TH:edge and correlation} shows the impacts of the correlations and prior knowledge on the privacy leakage of the aggregation of two correlated tuples, which correspond to the different cases in Fig.~\ref{fig:inferenceresults}.


 Combining Theorem \ref{TH:edge and correlation} with Eq.\;(\ref{Eq:privacy of two adjacent nodes}) and Eq.\;(\ref{Eq:IR}), we note that the weaker adversary causes higher privacy leakage when the tuples are positively correlated because more unknown tuples with positive correlations means a greater sensitivity to the query result. However, when tuples are negatively correlated, the weaker adversary does not cause less privacy leakage because more unknown tuples with negative correlations does not always mean smaller sensitivity or less privacy leakage.

 What about when the domain size of $x_i$ is greater than two?  Do the results in Theorem \ref{TH:edge and correlation} still hold? Regretfully, the answer is negative. Let the domain size of $x_i$ be $s$, then the number of $IR_{j,\mathcal{K}'}(x_{i,m},x_{i,n})$ is $\binom{s}{2}$. We cannot guarantee all these $IR_{j,\mathcal{K}'}(x_{i,m},x_{i,n})$ satisfy Theorem \ref{TH:edge and correlation}. Instead, we have the following analytical results.
 \begin{enumerate}
   \item If $\rho_{ij,\mathcal{K}'}>0$, at least one $IR_{j,\mathcal{K}'}(x_{i,m},x_{i,n})\in(0,1]$;
   \item if $x_i$ and $x_j$ are independent, all $IR_{j,\mathcal{K}'}(x_{i,m},x_{i,n})=0$;
   \item if $\rho_{ij,\mathcal{K}'}<0$, at least one $IR_{j,\mathcal{K}'}(x_{i,m},x_{i,n})\in[-1,0)$.
 \end{enumerate}
 Therefore, combining the above analytical results with Eqs.\;(\ref{Eq:IR}) and (\ref{Eq:privacy of two adjacent nodes}), we also conclude that the weaker adversary can cause higher privacy leakage when $x_i$ and $x_j$ are positively correlated. The prior knowledge has no impact on privacy leakage when $x_i$ and $x_j$ are independent, which also corresponds to Theorem \ref{TH:prior knowledge has no impact}. However, we cannot derive a deterministic relationship between the privacy leakage and prior knowledge if the tuples are negatively correlated. In this situation, we have to use Eq.\;(\ref{Eq:privacy of two adjacent nodes}) to determine their relationship.


 \subsection{Privacy Leakage Formulation}
 \label{Subsection:discrete-chainrule}

 In this subsection, we introduce how to compute the node value, which represents the privacy leakage caused by the adversary with prior knowledge in the WHG. As mentioned in Subsection \ref{Subsection:discrete-WHG}, the computation relies on two steps. One step computes the node values in the first layer; the other is a chain rule. We first present how to compute the node values in the first layer.
 \begin{thm}\label{TH:privacy of first level}
   \emph{For a database $\mathbf{x}$ which has $n$ tuples and follows the joint distribution $\theta$, the values of the nodes in the first layer are $l_{\mathcal{A}_{i,[n]\backslash\{i\}}}={LS_i(f)}/{\lambda}, i\in[n]$, where $LS_i(f)$ is the local sensitivity, and $\lambda$ is the parameter in the Laplace mechanism.}
 \end{thm}

 \begin{proof}
 Based on the definition of PDP, $\forall i\in [n]$, we have
 \begin{equation*}
 \begin{split}
   l_{\mathcal{A}_{i,[n]\backslash\{i\}}}&=\sup_{r,x_i,x_i'}\log\frac{\Pr(r|x_i,\mathbf{x}_{-i})}{\Pr(r|x_i',\mathbf{x}_{-i})}\\
   &=\sup_{r,x_i,x_i'}\log\frac{\exp(-|r-f(\mathbf{x})|/\lambda)}{\exp(-|r-f(\mathbf{x}')|/\lambda)}\\
   &\leq \sup_{x_i,x_i'}{|f(\mathbf{x})-f(\mathbf{x}')|}/{\lambda}=LS_i(f)/\lambda.
 \end{split}
 \end{equation*}
 \end{proof}

 Theorem \ref{TH:privacy of first level} demonstrates that the joint distribution, which represents the correlation, has no impact on privacy leakage when the adversary has the strongest prior knowledge. On the basis of Theorem \ref{TH:privacy of first level}, we deduce the values of the nodes in the second layer through Theorem \ref{TH:privacy of two adjacent nodes}. Similarly, we can obtain the values of the nodes in layer $k+1$ by layer $k$ according to Theorem \ref{TH:privacy of two adjacent nodes}. Finally, we can obtain all nodes' values. In particular, the following theorem presents a solution to computing the privacy leakage of a certain node in WHG.

 \begin{thm}\label{TH:chain rule}{(Chain Rule)}
   \emph{For a node $(i,\mathcal{K})$ in the layer $k+1$, there exists a path from the bottom node $[n]$ to the node $(i,\mathcal{K}),\mathcal{K}=[n]\backslash\{i,j_1,\cdots,j_k\}$. From layer 1 to layer $k+1$, $(i,[n]\backslash\{i\})$, $(i,[n]\backslash\{i,j_1\}),\cdots,(i,[n]\backslash\{i,j_1,\cdots,j_k\})$ are all the nodes in the path. Then, the privacy leakage of the node $(i,\mathcal{K})$ corresponding to this path is
   \begin{equation}\label{Eq:chain rule}
     \begin{split}
     l_{\mathcal{A}_{i,\mathcal{K}}}=&|\ldots||LS_i(f)/\lambda+IC_{ij_1,[n]\backslash\{i,j_1\}}|+IC_{ij_2,[n]\backslash\{i,j_1,j_2\}}|\\
     &+\cdots+IC_{ij_k,[n]\backslash\{i,j_1,\cdots,j_k\}}|,
     \end{split}
   \end{equation}
   where $|\ldots||$ denotes $k$-fold absolute value operation, and $k$ is the length of the path.}
 \end{thm}
 \begin{proof}
 The result can be obtained by using Theorem \ref{TH:privacy of two adjacent nodes}. In a path from the bottom to the top, there are $k+1$ nodes and $k$ edges, each of which consists of two nodes in the adjacent layers. The chain rule can be obtained by applying Theorem \ref{TH:privacy of two adjacent nodes} on all $k$ edges in a path.
 \end{proof}
 Theorem \ref{TH:chain rule} shows the computational process for a path from the bottom node to the given node. If there exist multiple paths, we should compute the value of each path by using Eq.\;(\ref{Eq:chain rule}), and then choose the minimum as the node's value.
 There are three factors that can impact privacy leakage. The length of the path in Eq.\;(\ref{Eq:chain rule}), which represents the amount of prior knowledge. To highlight the other two factors, according to Eq.\;(\ref{Eq:IR}), we rewrite Eq.\;(\ref{Eq:chain rule}) as follows
 \begin{equation}\label{Eq:chain rule-decomposition form}
     \begin{split}
     l_{\mathcal{A}_{i,\mathcal{K}}}&=
     \left|\ldots\left|\left|\frac{LS_i(f)}{GS(f)}+IR_{ij_1,[n]\backslash\{i,j_1\}}\frac{LS_{j_1}(f)}{GS(f)}\right.\right.\right|\\
     &\left.+\cdots+IR_{ij_k,[n]\backslash\{i,j_1,\cdots,j_k\}}\frac{LS_{j_k}(f)}{GS(f)}\right|\frac{GS(f)}{\lambda}.
     \end{split}
   \end{equation}

 According to Eq.\;(\ref{Eq:chain rule-decomposition form}), we can see that PDP is superior to group differential privacy in terms of calculating an accurate privacy leakage for the adversary with specific prior knowledge. Particularly, according to Theorem \ref{TH:edge and correlation}, we have $IR_{ij,\mathcal{K}}\in[-1,1],\forall \mathcal{K}\subseteq [n]\backslash\{i\}$. By setting all $IR_{ij,\mathcal{K}}=1$ in Eq.\;(\ref{Eq:chain rule-decomposition form}), we have
 \begin{align}
   l_{\mathcal{A}_{i,\mathcal{K}}}&\leq \sum_{j\in\{i,j_1,\cdots,j_k\}}LS_j(f)/\lambda\label{Eq:PDP is superior to group}\\
   &\leq (k+1)GS(f)/\lambda.\label{Eq:PDP is superior to sequential}
 \end{align}
 Eqs.\;(\ref{Eq:PDP is superior to group}) and (\ref{Eq:PDP is superior to sequential}) show that the privacy leakage under PDP is more accurate than group differential privacy, which is simply derived from the sequential composition theorem. In addition, when the edge values in the WHG are all greater or less than zero, we can deduce some special results in the following corollary.
 \begin{cor}\label{Cor:privacy leakage in sepcial cases}
   \begin{enumerate}
     \item When all $IC_{ij,\mathcal{K}}=1$, the PDP degrades to group differential privacy.
     \item When all $IC_{ij,\mathcal{K}}>0$, the maximal privacy leakage is obtained at the top layer, i.e., the weakest adversary causes the highest privacy leakage.
     \item When all $IC_{ij,\mathcal{K}}<0$, the maximal privacy leakage is obtained in the bottom layer, i.e., the strongest adversary causes the highest privacy leakage.
   \end{enumerate}
 \end{cor}
 Case 1) can be derived from Eq.\;(\ref{Eq:PDP is superior to group}) directly. Additionally, it is easy to prove Cases 2) and 3) by using Eq.\;(\ref{Eq:chain rule-decomposition form}) and summing the nodes' values in layer order.

 Based on Corollary \ref{Cor:privacy leakage in sepcial cases}, we can easily compute the privacy leakage for these special cases. For example, in Case 2), the privacy leakage increases with the layer number. However, in general cases, when WHG has both positive and negative edges, we have to traverse the whole WHG to compute the privacy leakage.

 \subsection{Algorithms for Computing Privacy Leakage}
 \label{Subsection:discrete-algorithm}

 For a given database $\mathbf{x}$ with $n$ tuples, the number of edges is no fewer than the number of nodes $n2^{n-1}$. Therefore, it is intractable to traverse the WHG when the number of tuples is large.
 We first use the full-space-searching algorithm to compute the least upper bound of privacy leakage and then propose a heuristic fast-searching algorithm to reduce the calculation time by limiting the searching space.

 In the full-space-searching algorithm, we first initialize the value of the nodes in the first layer by Theorem \ref{TH:privacy of first level} (line 1). Then, we generate nodes in layer $k+1$ by using the chain rule (Theorem \ref{Eq:chain rule}) based on the edges' value (Eq.\;(\ref{Eq:privacy of two adjacent nodes})) between layers $k$ and $k+1$ (lines 3-10). Note that for a given node in layer $k+1$, there may exist multiple paths from the nodes in layer $k$ to the given node. As mentioned previously, we need to retain the minimal value computed from multiple paths as the node value (line 8). Finally, we obtain the maximal privacy leakage of all nodes in the WHG (line 11).
 \begin{algorithm}[!htb]
   \caption{Full-Space-Searching}
   \label{Alg:original algorithm of hiergraph}
   \begin{algorithmic}[1]
   \Require Database $\{x_1,x_2,\cdots,x_n\}$, joint distribution $\theta$
   \Ensure Privacy Leakage $l$
   \State Generate nodes $(i,[n]\backslash\{i\})$ in the first layer, set $l_{\mathcal{A}_{i,[n]\backslash\{i\}}}=LS_i(f)/\lambda$ and $l_1=GS(f)/\lambda$;
   \State Denote all nodes in the first layer as $\mathcal{N}_1$;
   \For {$k$=2 to $n$}
        \For {each node $(i,\mathcal{K})\in\mathcal{N}_{k-1}$}
            \State Generate node $(i,\mathcal{K}')$ by subtracting $\{j\}$ from $\mathcal{K}$;
            \State Compute $l_{\mathcal{A}_{i,\mathcal{K}'}}=|l_{\mathcal{A}_{i,\mathcal{K}}}+IC_{ij,\mathcal{K}'}|$;
        \EndFor
        \State Detect the repeated nodes with the same attack tuple and prior knowledge in layer $k$; only retain the node with the minimal privacy leakage;
        \State \Return $l_k=\max_{(i,\mathcal{K}')\in\mathcal{N}_k}\{l_{\mathcal{A}_{i,\mathcal{K}'}}\}$;
   \EndFor
   \State \Return $l=\max_{k\in[n]}\{l_k\}$;
   \end{algorithmic}
 \end{algorithm}

 \begin{pro}\label{TH:complexity of full space search}
 The time complexity of the full-space-searching algorithm is $O(n^4 2^{n-1})$.
 \end{pro}

 \begin{proof}
 There are two steps to obtain the value of the nodes in the layer $k+1$ from the value of the nodes in the layer $k$. One is to first obtain new nodes in layer $k+1$ by removing one tuple from the prior knowledge of the nodes in layer $k$. The second step is to sort and remove the repeated nodes with the same attack tuple and prior knowledge in layer $k+1$. There are $n\binom{n-1}{n-k}$ nodes in layer $k$, so the number of nodes after the first step would be $n\binom{n-1}{n-k}(n-k)$, denoted as $t_k$. The time complexity after the second step is $\Theta(t_k\log t_k)$. We note that $\sum_{k=1}^{n-1}t_k\leq n^22^{n-1}$. Summing from $k=1$ to $n-1$, the time complexity of the algorithm is
 \begin{equation}
   \begin{split}
     \sum_{k=1}^{n-1}(t_k+t_k\log t_k)&\leq \sum_{k=1}^{n-1}t_k+\sum_{k=1}^{n-1}t_k \sum_{k=1}^{n-1}\log t_k\\
     &\leq n^22^{n-1}(1+2n^2)\leq 3n^4 2^{n-1}.
   \end{split}
 \end{equation}
 \end{proof}

 As we can see, considerable time will be required to generate new nodes and to remove repeating nodes in the full-space searching algorithm. In addition, the time complexity grows exponentially with the number of tuples $n$. To reduce the computational time complexity, a fast-searching algorithm is proposed to search a subspace of the original full space with a little sacrifice of the accuracy. Specifically, we only use the top $n$ largest nodes in layer $k$ to generate layer $k+1$.

 \begin{pro}\label{TH:complexity of fast search}
  The time complexity of the fast-searching algorithm is $O(n^4)$.
 \end{pro}

 \begin{proof}
 According to the fast-searching algorithm, there are, at most, $n$ tuples in layer $k$. After the subtraction operation, there are at most $n(n-k)$ tuples, denoted as $t_k$. The rest of this proof is the same as that of proposition \ref{TH:complexity of full space search}.
 \end{proof}

 \begin{algorithm}[!htb]
   \caption{Fast-Searching}
   \label{Alg:fast algorithm of hiergraph}
   \begin{algorithmic}[1]
   \Require Database $\{x_1,x_2,\cdots,x_n\}$, joint distribution $\theta$
   \Ensure Privacy Leakage $l$
   \State Initialize nodes in layer $1$;
   \For {$k$=2 to $n$}
   \State Generate the nodes in layer $k$;
   \State Detect the repeated nodes and retain the minimum node;
   \State Retain the top $\min\{n,\#(\mathcal{N}_{k-1})\}$ largest nodes;
   \State return $l_k$;
   \EndFor
   \State \Return {$l=\max_{k\in[n]}\{l_k\}$};
   \end{algorithmic}
 \end{algorithm}

 \section{Gaussian Model-based Analysis for Continuous Data}
 \label{Section:continuous}

 In this section, we further discuss the impacts of correlation and prior knowledge for the continuous-valued data. In Subsection \ref{Subsection:continuous-Gaussianmodel}, we first explain why the WHG is not suitable for the continuous-valued database. Then, we introduce some properties of the multivariate Gaussian distribution. In Subsection \ref{Subsec:continuous-privacyformulation}, we identify an explicit formula to compute the privacy leakage of the multivariate Gaussian model.

 \subsection{Multivariate Gaussian Model}
 \label{Subsection:continuous-Gaussianmodel}

The necessity to separate the continuous situation from the discrete situation is that the computation method used in Section \ref{Section:WHG for discrete data} is no longer sustainable. In Section \ref{Section:WHG for discrete data}, we investigate how correlation and prior knowledge can impact privacy leakage. Based on the proposed WHG, we deduce the chain rule to compute privacy leakage. One crucial step is to compute the edge value in the WHG. Discrete-valued data can be achieved by using Eqs.\;(\ref{Eq:IC}) and (\ref{Eq:privacy of two adjacent nodes}), which requires enumerating all the different pairs of value $x_{i,m}$ and $x_{i,n}$ in the domain. Obviously, it is impossible for continuous-valued tuples with an unbounded domain. To deal with this issue, we should clarify the joint distribution. Therefore, although the analytical results in Section \ref{Section:WHG for discrete data} still holds for both continuous-valued data; the edges' value cannot be directly computed as discrete data.

 For continuous data, a common solution is to accurately identify the global sensitivity via bounding the range (i.e., domain) of the tuples \cite{Kifer2014Pufferfish,Yang2015Bayesian}. Otherwise, the privacy leakage would be overestimated, and the unboundedness would destroy the utility of privacy-preserving results. Therefore, by bounding the range of $x_i$ as $|x_i-x_i'|\leq M,r$, Eq.\;(\ref{Eq:definition of PDP}) becomes
 \begin{align}\label{Eq:PDP with bounded}
      l_{\mathcal{A}_{i,\mathcal{K}}}(\theta)=\sup_{|x_i-x_i'|\leq M,r}\log\frac{\mathrm{Pr}(r\in S|x_i,\mathbf{x}_{\mathcal{K}})}{\mathrm{Pr}(r\in S|x_i',\mathbf{x}_{\mathcal{K}})}.
 \end{align}
 Different from the sum operation in computing probability in Section \ref{Section:WHG for discrete data}, we use integrate to compute the probability for continuous data. That is
 \begin{align*}
   \Pr(r|x_i,\mathbf{x}_{\mathcal{K}})=\int_{x_j}\Pr(x_j|x_i,\mathbf{x}_{\mathcal{K}})\Pr(r|x_j,x_i,\mathbf{x}_{\mathcal{K}})\mathrm{d} x.
 \end{align*}


 Here, we choose the multivariate Gaussian distribution (denoted as MGD) to describe the database $\mathbf{x}$ since most of the continuous data can be well modeled by MGD. For a database $\mathbf{x}$ with $n$ tuples, $\mathbf{x}=\{x_1,\cdots,x_n\}$, $\boldsymbol{\mu}$ is the expectation vector, and $\boldsymbol{\Sigma}=(\rho_{ij})_{n\times n}$ is the covariance matrix. If $\rho_{ij}>0$, $x_i$ and $x_j$ are positively correlated. If $\rho_{ij}<0$, $x_i$ and $x_j$ are negatively correlated. If $\rho_{ij}=0$, $x_i$ and $x_j$ are independent. $\mathbf{x}$ follows the MGD if the density function of $\mathbf{x}$ is
 \begin{align*}
   f(\mathbf{x})=(2\pi)^{-\frac{n}{2}}\left|\mathbf{\Sigma}\right|^{-\frac{1}{2}}
   \exp\left(-\frac{1}{2}(\mathbf{x}-\boldsymbol{\mu})'\mathbf{\Sigma}^{-1}(\mathbf{x}-\boldsymbol{\mu})\right),
 \end{align*}
 and denote $\mathbf{x}\sim N_n(\boldsymbol{\mu},\mathbf{\Sigma})$. If $\mathbf{x}$ is blocked as $\{\mathbf{x}_1,\mathbf{x}_2\}$, then $\boldsymbol{\mu},\mathbf{\Sigma}$ can be written as
 \begin{align*}
   \boldsymbol{\mu}=[\boldsymbol{\mu}_1,\boldsymbol{\mu}_2]',\boldsymbol{\Sigma}=\left[  \begin{array}{cc}
    \boldsymbol{\Sigma}_{11} & \boldsymbol{\Sigma}_{12} \\
    \boldsymbol{\Sigma}_{21} & \boldsymbol{\Sigma}_{22} \\
  \end{array}\right]
 \end{align*}
 The following lemma shows the properties of the MGD.

 \begin{lem}\cite{Eaton1983Multivariate}\label{Lem:property of Gaussian}
  \emph{Given the $n$-dimensional variable $\mathbf{x}=\{\mathbf{x}_1,\mathbf{x}_2\}$ follows the multivariate Gaussian distribution $N_n(\boldsymbol{\mu},\boldsymbol{\Sigma})$, $\mathbf{x}_1\in \mathbb{R}^p,\mathbf{x}_2\in \mathbb{R}^{n-p}$.}
 \begin{enumerate}
   \item \emph{The distribution of $\mathbf{x}_1$ given $\mathbf{x}_2$ follows the $p$-dimensional Gaussian distribution $N_p(\boldsymbol{\mu}_{1|2},\mathbf{\Sigma}_{1|2}),\mathbf{\Sigma}_{22}\succ 0$, with
     \begin{align}
      \boldsymbol{\mu}_{1|2}&=\boldsymbol{\mu}_1+\mathbf{\Sigma}_{12}\mathbf{\Sigma}_{22}^{-1}(\mathbf{x}_2-\boldsymbol{\mu}_2),\label{Eq:mean of Gaussian}\\
      \mathbf{\Sigma}_{1|2}&=\mathbf{\Sigma}_{11}-\mathbf{\Sigma}_{12}\mathbf{\Sigma}_{22}^{-1}\mathbf{\Sigma}_{21}\label{Eq:variance of Gaussian}.
     \end{align}}
   \item \emph{For any nonzero vector $\mathbf{a}\in \mathbb{R}^n$,
    \begin{equation}\label{Eq:sum of Gaussian}
      \mathbf{a}^\top\mathbf{x}\sim N_1(\mathbf{a}^\top\boldsymbol{\mu},\mathbf{a^\top\Sigma a}).
    \end{equation}}
 \end{enumerate}
 \end{lem}

\subsection{Privacy Leakage Computation}
\label{Subsec:continuous-privacyformulation}

 As we can see from Eq.~(\ref{Eq:PDP with bounded}), the key to computing privacy leakage is computing the conditional probability when $\mathbf{x}$ follows the MGD.
 Let $\mathbf{x}_2=\{x_i,\mathbf{x}_{\mathcal{K}}\}$, $\mathbf{x}_1=\mathbf{x}\backslash\{x_i,\mathbf{x}_{\mathcal{K}}\}=\mathbf{x}_{\mathcal{U}}$, we have
 \begin{align*}
   \mathrm{Pr}(r|x_i,\mathbf{x}_{\mathcal{K}})=\mathrm{Pr}(r|\mathbf{x}_2)
   =\int_{\mathbf{x}_1}\mathrm{Pr}(\mathbf{x}_1|\mathbf{x}_2)\mathrm{Pr}(r|\mathbf{x})\mathrm{d}\mathbf{x}_1,
 \end{align*}
 in which $\mathrm{Pr}(\mathbf{x}_1|\mathbf{x}_2)$ can be obtained by Lemma \ref{Lem:property of Gaussian}, and $\mathrm{Pr}(r|\mathbf{x})$ can be calculated according to the Laplace mechanism. In addition, according to Eq.~(\ref{Eq:PDP with bounded}), when the attack object $x_i$ is replaced with $x_i'$, we have to compute
 \begin{align}\label{Eq:conditional probability of Gaussian}
   \mathrm{Pr}(r|x_i',\mathbf{x}_{\mathcal{K}})
   =\int_{\mathbf{x}_1}\mathrm{Pr}(\mathbf{x}_1|\mathbf{x}_2')\mathrm{Pr}(r|\mathbf{x}')\mathrm{d}\mathbf{x}_1,
 \end{align}
  where $\mathbf{x}_2'=\{x_i',\mathbf{x}_{\mathcal{K}}\}$, $\mathbf{x}'=\{x_i',\mathbf{x}_{\mathcal{K}},\mathbf{x}_{\mathcal{U}}\}$.

  However, it is difficult to directly calculate the above probability formula under the assumption that $|x_i-x_i'|\leq M$ for all possible pairs of $x_i$, and $x_i'$. Fortunately, the next lemma shows that we can combine the $\mathrm{Pr}(\mathbf{x}_1|\mathbf{x}_2')\mathrm{Pr}(r|\mathbf{x}')$ into a uniform expression, which is useful for computation.


 \begin{lem}\cite{Yang2015Bayesian}\label{Lem:lemma for prove Gaussian model}
 \emph{ Let $G(x;b)$ be a function on $x\in \mathbb{R}$, with parameter $b>0$,
 \begin{align}
   G(x;b)=e^x \left( 1-\Phi(\frac{x}{b}+b) \right)+e^{-x}\Phi(\frac{x}{b}-b),
 \end{align}
 where $\Phi(x)$ is the cumulative distribution function of a standard Gaussian distribution. Then, $\frac{\partial\log G(x;b)}{\partial x}$ is monotonically decreasing with respect to $x$ and
 \begin{align*}
    \lim\limits_{x\rightarrow -\infty}\frac{\partial\log G(x;b)}{\partial x}=1;\lim\limits_{x\rightarrow +\infty}\frac{\partial\log G(x;b)}{\partial x}=-1.
  \end{align*}}
 \end{lem}

 Based on Lemmas \ref{Lem:property of Gaussian} and \ref{Lem:lemma for prove Gaussian model}, we propose the next formula to compute the privacy leakage of an adversary $(i,\mathcal{K})$ directly.

 \begin{thm}\label{TH:privacy of Gaussian model}
  \emph{Given $\mathbf{x}$ follows the multivariate Gaussian distribution $N_n(\boldsymbol{\mu,\Sigma})$, $f(\mathbf{x})=\sum_{i\in[n]} x_i$ is a general sum query on $\mathbf{x}$. Let $\mathcal{M}$ be Laplace mechanism with the perturbed output $r=\mathcal{M}(\mathbf{x})=f(\mathbf{x})+z$, where $z\sim Lap(\lambda)$. Given an adversary $\mathcal{A}_{i,\mathcal{K}}$ with $\theta\in\Theta$ and $|x_i-x_i'|\leq M$, then the privacy leakage can be represented as
  \begin{align}
    l_{\mathcal{A}_{i,\mathcal{K}}}(\theta)=\frac{M}{\lambda}\left|1+\mu_{0i}\right|.\label{Eq:privacy of Gaussian model}
  \end{align}
  where $\mu_{0i}$ is the coefficient of $x_i$ in the expansion of $\mu_0$ in Eq.\;(\ref{Eq:expansion of mu_0 in Gaussian}).}
 \end{thm}

 \begin{proof}
   See Appendix \ref{Proof:privacy of Gaussian model}.
 \end{proof}
 Theorem~\ref{TH:privacy of Gaussian model} shows the impacts of the correlation and prior knowledge on privacy leakage for continuous data under the multivariate Gaussian distribution.
 For the given $M$ and $\lambda$, we can see that the privacy leakage is determined by $\mu_{0i}$, which is related to the covariance matrix $\mathbf{\Sigma}$ and prior knowledge $\mathcal{K}$. $\mu_{0i}$ is the coefficient of $x_i$ in $\mu_0$, which can be obtained from Eqs.\;(\ref{Eq:mean of Gaussian}) and (\ref{Eq:sum of Gaussian}). The details can be found in the proof of Theorem \ref{TH:privacy of Gaussian model} (Appendix \ref{Proof:privacy of Gaussian model}). In the analysis of discrete-valued data without a concrete expression of the data distribution, the chain rule is proposed to compute the privacy leakage. However, for continuous data, based on the assumption of the MGD, we can compute the privacy leakage of an adversary directly without considering every two adjacent adversaries. For a special case that considers the weakest adversary, the privacy leakage has the following explicit form.

 \begin{cor}\label{cor:special case of Gaussian}
  \emph{For an $n$-dimensional Gaussian distribution $N_n(\boldsymbol{\mu,\Sigma}),\mathbf{\Sigma}=(\rho_{ij})_{n\times n}$, the privacy leakage of the weakest adversary is
    $l_{\mathcal{A}_{i,\phi}}=|1+\rho_{ii}^{-1}\sum_{j\neq i}\rho_{ij}|M/\lambda, i\in [n]$.}
 \end{cor}
 \begin{proof}
 In such case, the $\mathbf{x}_2=\{x_i\}$, and $\mathbf{x}_1=\mathbf{x}_{-i}$. According to Lemma \ref{Lem:property of Gaussian}, $\boldsymbol{\mu}_{1|2}=(\mu_1,\cdots,\mu_{i-1},\mu_{i+1},\cdots,\mu_n)^\top+(\rho_{i1}+\cdots+\rho_{i,j-1},
   \rho_{i,j+1},\cdots,\rho_{in})^\top\cdot\rho_{ii}^{-1}\cdot(x_i-\mu_i).$
 By the definition of $\mu_0$, we have
 \begin{align*}
     \mu_0=\mathbf{1}^\top\boldsymbol{\mu}_{1|2}=\sum_{j\neq i}\mu_j-\mu_i\sum_{j\neq i}\rho_{ij}+x_i\rho_{ii}^{-1}\sum_{j\neq i}\rho_{ij}.
 \end{align*}
 Here, $\mu_{0i}$, the coefficient of $x_i$ in $\mu_0$, is $\rho_{ii}^{-1}\sum_{j\neq i}\rho_{ij}$.
 Then, we complete the proof by applying Theorem \ref{TH:privacy of Gaussian model}.
 \end{proof}

 The proof of Corollary \ref{cor:special case of Gaussian} demonstrates a special case of how to compute the coefficient $\mu_{0i}$ for the weakest adversary in Eq.\;(\ref{Eq:privacy of Gaussian model}). As specific cases of Theorem \ref{TH:privacy of Gaussian model}, Corollary \ref{cor:special case of Gaussian} above demonstrates that privacy leakage of the weakest adversary has an explicit relationship to the data correlation. That is, privacy leakage of the weakest adversary depends on the summation of all covariances connecting to $x_i$, which represents the data correlations. Next, the example shows the impacts of the correlations and prior knowledge for tuples that follow a two-dimensional Gaussian distribution.

 Example 5 \textit{Consider a continuous-valued database $\mathbf{x}=\{x_1,x_2\}$; the expectation vector and variance matrix of $\mathbf{x}$ are $\boldsymbol{\mu}=[0,0]$, $\boldsymbol{\Sigma}=
 \left(
 \begin{array}{cc}
      1 & \rho_{12} \\
      \rho_{12} & 1 \\
    \end{array}
  \right)$, respectively.
 $\rho_{12}\in[-1,1]$ is the correlation coefficient of $x_1$ and $x_2$. From Corollary \ref{cor:special case of Gaussian}, $l_{\mathcal{A}_{1,\emptyset}}=|1+\rho_{12}|M/\lambda$. From the definition of PDP, we have $l_{\mathcal{A}_{1,\{2\}}}=M/\lambda$. If $\rho_{12}>0$, then $l_{\mathcal{A}_{1,\emptyset}}>l_{\mathcal{A}_{1,\{2\}}}$. This means that when the correlation is positive, a weak adversary has more privacy leakage gain than a strong adversary. If $\rho_{12}<0$, then $l_{\mathcal{A}_{1,\emptyset}}<l_{\mathcal{A}_{1,\{2\}}}$. This means that when the correlation is negative, the strong adversary has more privacy gain. These results are also consistent with Examples 2 and 3, which are discrete-valued data.}



\section{Numerical Simulations}
\label{Sec:Experiments}

 In this section, we conducted extensive experiments to demonstrate the impact of prior knowledge and data correlations on privacy leakage, and validate the effectiveness and efficiency of our proposed algorithms for computing privacy leakage.

 \subsection{Simulations Setting}

 We synthesized a database with 15 tuples\footnote{As described in Subsection \ref{Subsection:preliminary-notations}, a tuple refers to an attribute in the database.}, in which the average Pearson correlation coefficient changes from $-0.8$ to $0.8$. For the discrete-valued database, we generated a corresponding WHG by assigning beta-distributed edges' value. For the continuous-valued database, we generated the covariance matrix with covariance $Cov(x_i,x_j)=1, (i\neq j)$ for a positive correlation, and $Cov(x_i,x_j)=-1, (i\neq j)$ for a negative correlation. We adjusted the principal diagonal element $Cov(x_i,x_i)$ to control the correlation coefficient.

 In our experiments, we considered an adversary who can infer information from a Laplace-mechanism-based privacy-preserving sum query on the database. The noise scale of the Laplace mechanism was set as $\lambda=1$, and the domain size of all tuples was set as 1. In all simulations, the prior knowledge was measured by the number of tuples compromised by the adversary, ranging from $14$ to $0$. Then, the privacy leakage the adversary caused was calculated according to our analytical results (Theorems \ref{TH:privacy of two adjacent nodes}, \ref{TH:chain rule}, \ref{TH:privacy of Gaussian model}) in Sections \ref{Section:WHG for discrete data} and \ref{Section:continuous}.

 \subsection{Simulation Results}
 For simplicity, let averCorr denote the average value of the edges in the WHG, and averCoeff denote the average value of the correlation coefficient in the MGD. averCorr and averCoef represent the correlation degree for discrete-valued and continuous-valued data, respectively. According to the structure of the WHG, the layer number in a WHG represents the number of unknown tuples for an adversary. Therefore, a smaller layer number means a stronger adversary with more prior knowledge and vice versa. Fig.\;\ref{Fig:privacy discrite positive} and Fig.\;\ref{Fig:privacy discrite negative} shows the privacy leakage of discrete-valued data. Fig.\;\ref{Fig:privacy continuous positive} and Fig.\;\ref{Fig:privacy continuous negative} shows the privacy leakage of continuous-valued data.


 \subsubsection{Privacy Leakage vs Correlation}
 \label{Subsec:numerical simulation-privacy vs correlation}
 This subsection investigates the impacts of data correlations on privacy leakage when the prior knowledge is fixed.

 Figs.\;\ref{Fig:privacy discrite positive}-\ref{Fig:privacy continuous negative} show that the privacy leakage remains unchanged with averCorr and averCoeff when the adversary has the strongest prior knowledge (layer number=1, i.e., fewest unknown tuples). This is because the uncertainty only occurs from the attack object and no information gain can be obtained from the correlations, which corresponds to our analysis in Theorem \ref{TH:privacy of first level}.


 Fig.\;\ref{Fig:privacy discrite positive} shows that the privacy leakage generally increases with averCorr when averCorr is positive, for discrete-valued data. The main reason is that with the increase in positive correlations, tuples are more likely to show the similar trends and the difference of the sum aggregation becomes much larger, from which the adversary could obtain more information gain of unknown tuples based on his prior knowledge (Theorem \ref{TH:edge and correlation}). Fig.\;\ref{Fig:privacy discrite negative} shows the similar results when averCorr is negative. Fig.\;\ref{Fig:privacy continuous positive} and Fig.\;\ref{Fig:privacy continuous negative} show similar results in the continuous data.


\begin{figure*}[htbp]
\begin{flushleft}
\subfigure[Positive correlation]{
\noindent\begin{minipage}[t]{0.23\linewidth}
\centering
\includegraphics[width=0.225\paperwidth,height=1.62in]{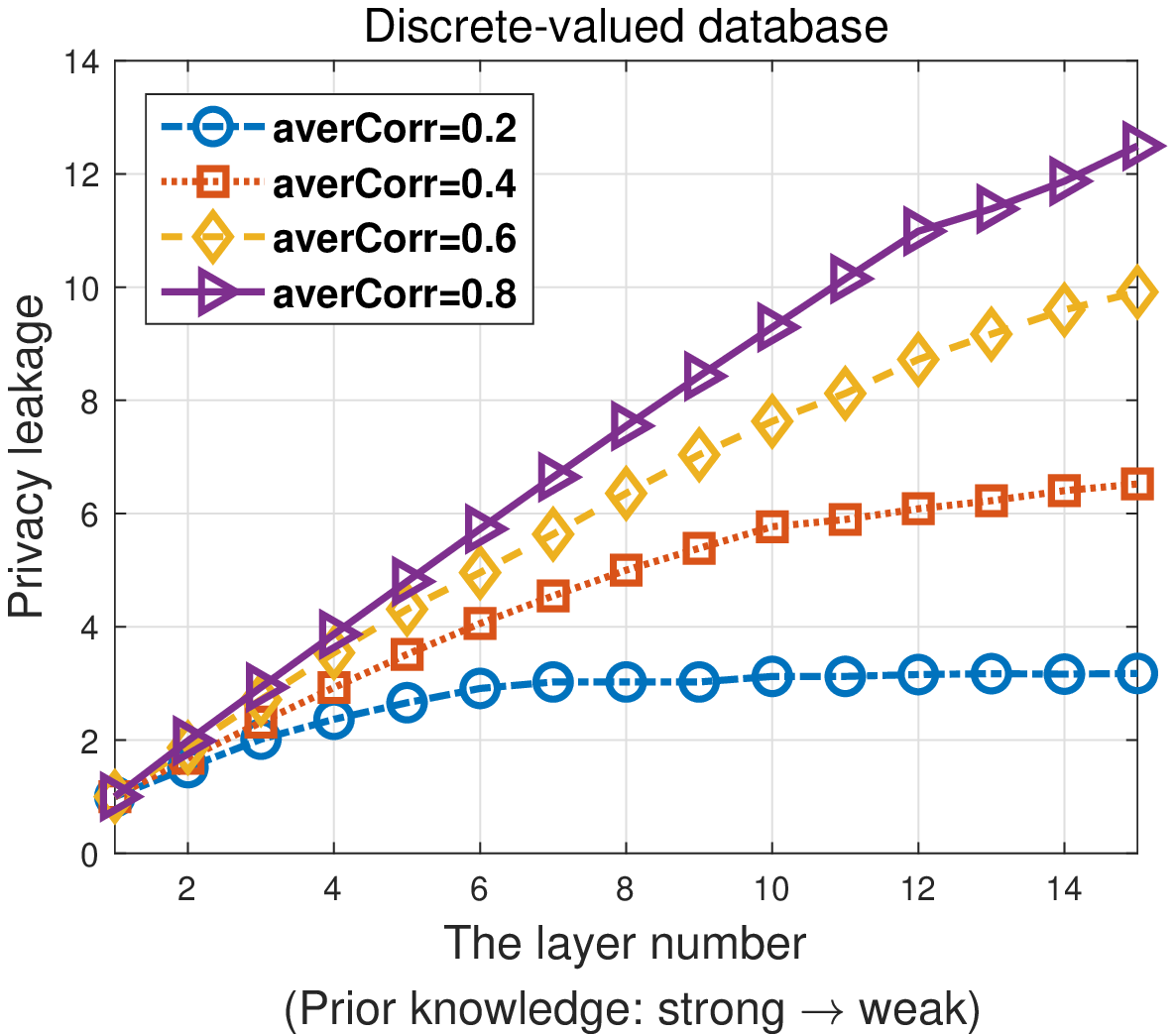}
\end{minipage}
\label{Fig:privacy discrite positive}
}
\subfigure[Negative correlation]{
\begin{minipage}[t]{0.23\linewidth}
\centering
\includegraphics[width=0.225\paperwidth,height=1.62in]{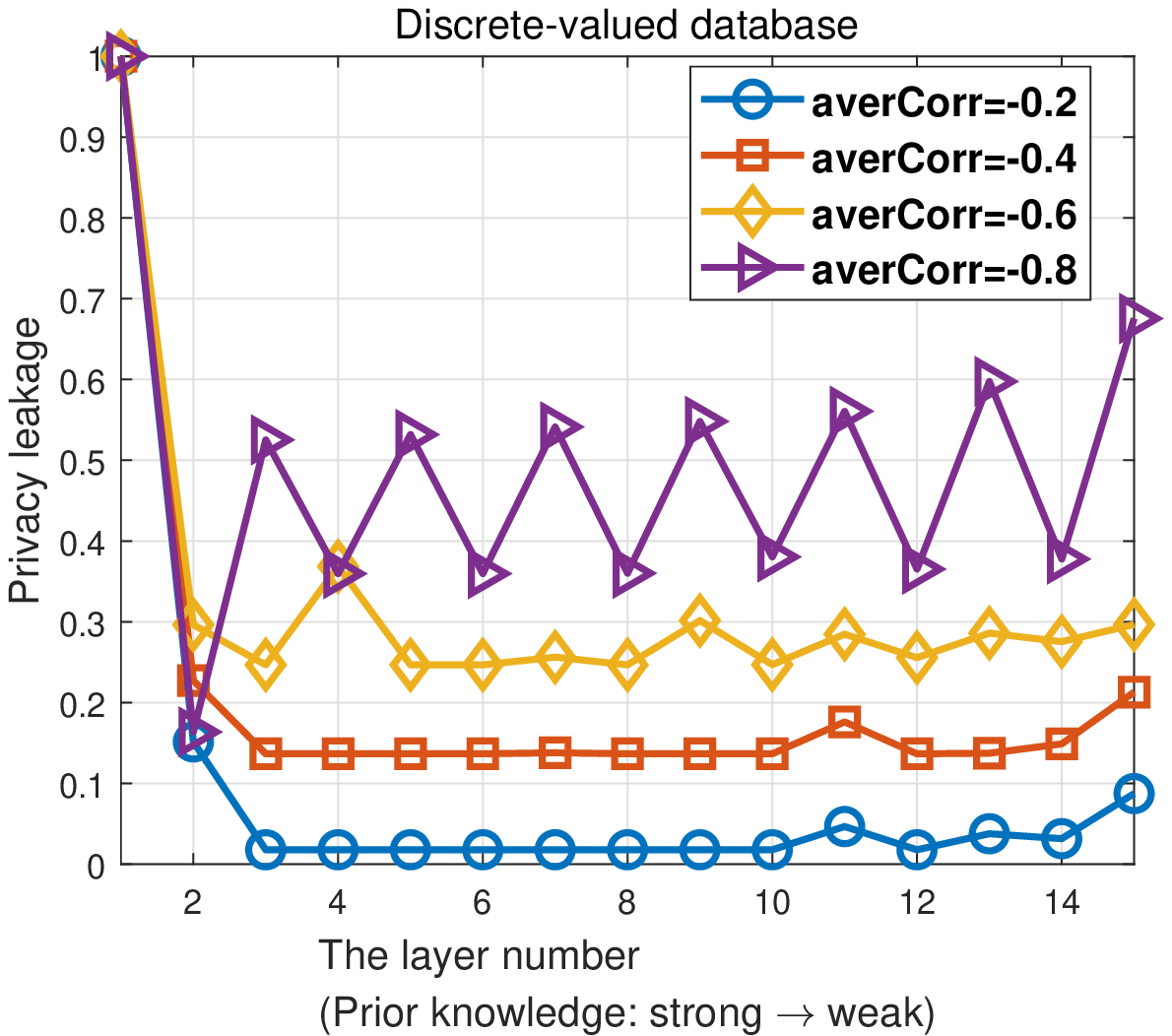}
\end{minipage}
\label{Fig:privacy discrite negative}
}
\subfigure[Positive correlation]{
\begin{minipage}[t]{0.23\linewidth}
\centering
\includegraphics[width=0.225\paperwidth,height=1.62in]{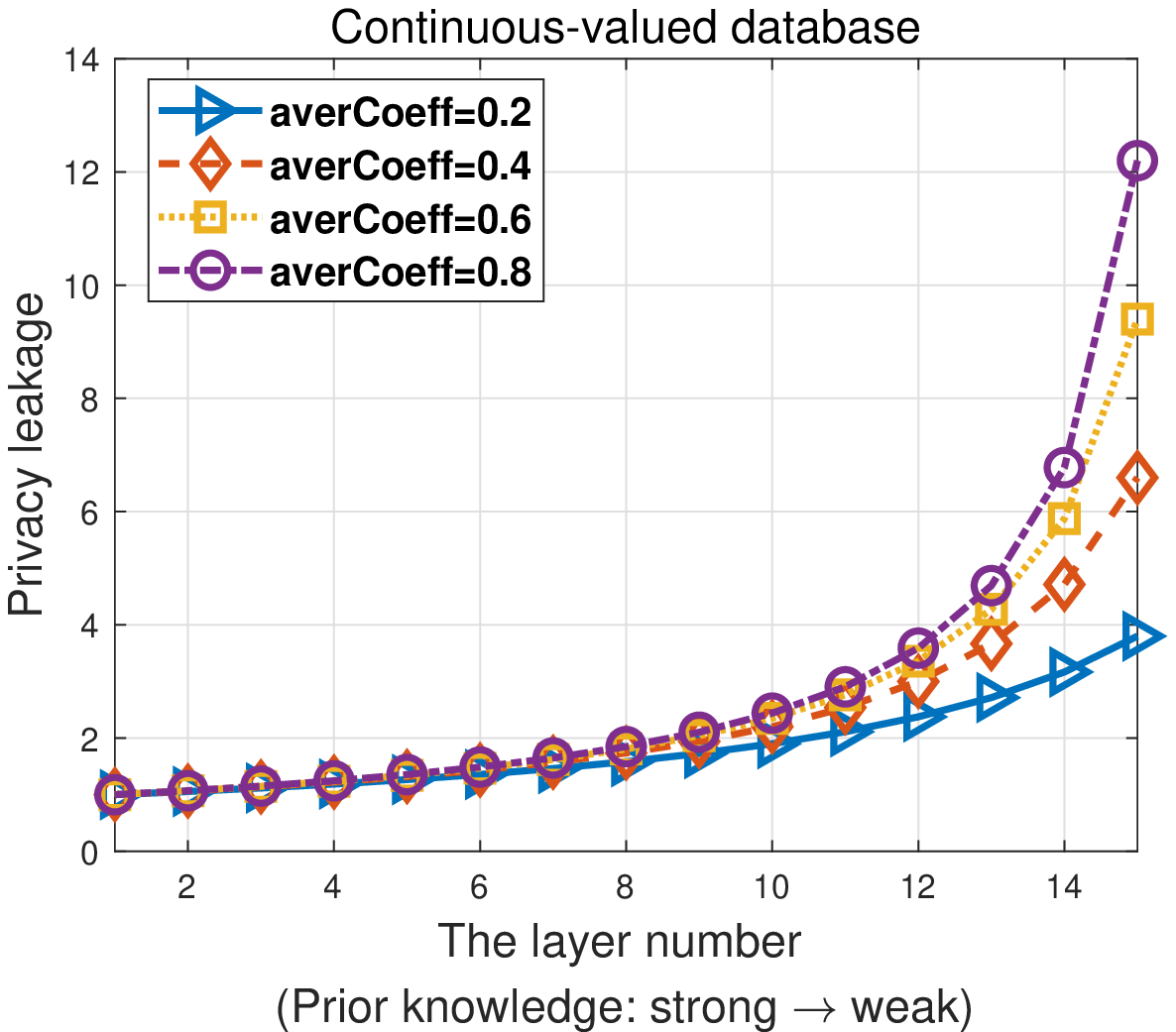}
\end{minipage}
\label{Fig:privacy continuous positive}
}
\subfigure[Negative correlation]{
\begin{minipage}[t]{0.23\linewidth}
\centering
\includegraphics[width=0.225\paperwidth,height=1.62in]{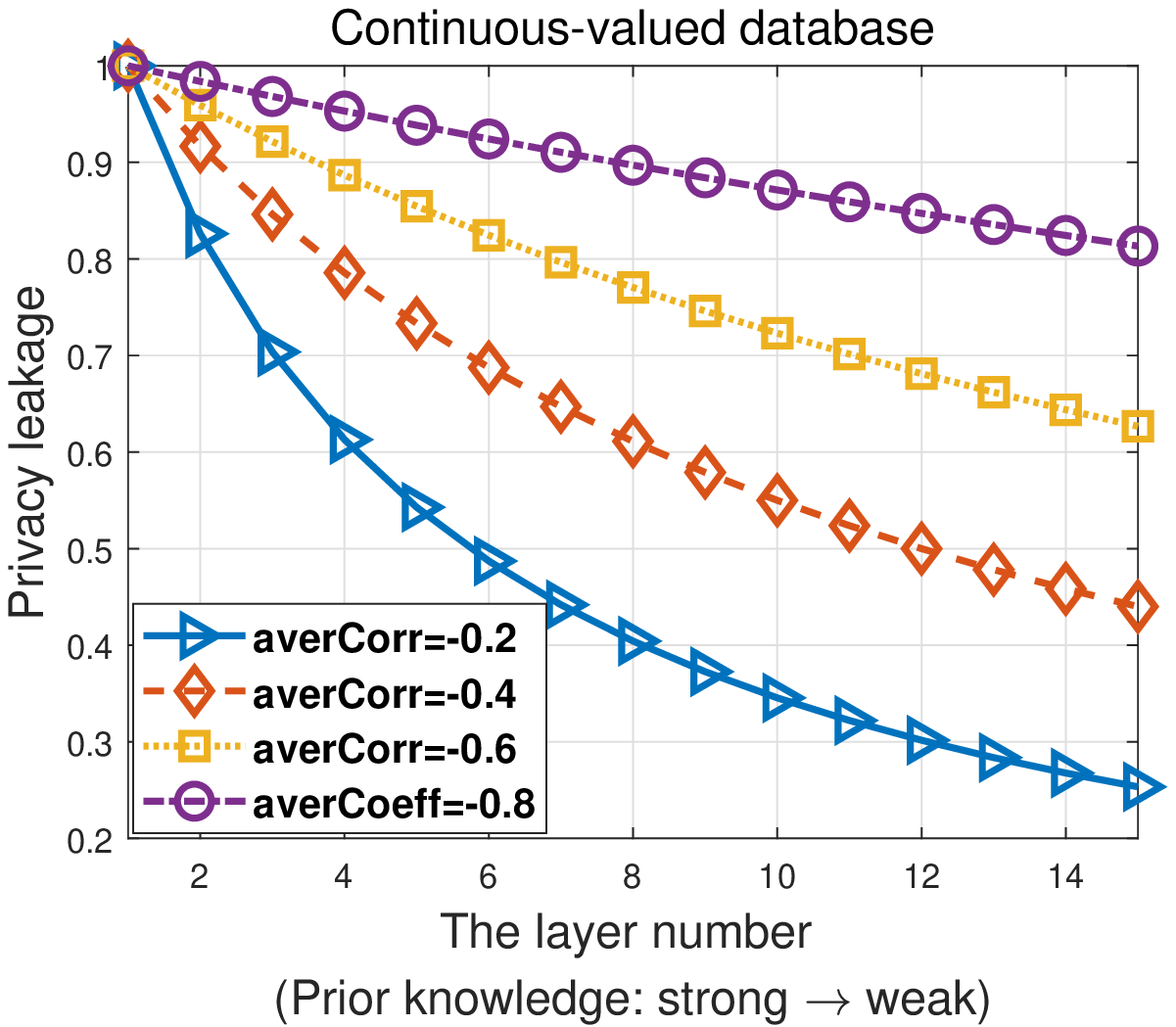}
\end{minipage}
\label{Fig:privacy continuous negative}
}
\caption{Privacy leakage vs. prior knowledge and data correlations. Figs.\;\ref{Fig:privacy discrite positive} and \ref{Fig:privacy discrite negative} show the results of the discrete-valued database. Figs.\;\ref{Fig:privacy continuous positive} and \ref{Fig:privacy continuous negative} show the results of the continuous-valued database. }
\end{flushleft}
\label{fig:privacy leakage of different database}
\end{figure*}

 \begin{figure*}[htbp]
\begin{flushleft}
\subfigure[averCorr=0.2]{
\begin{minipage}[t]{0.23\linewidth}
\centering
\includegraphics[width=0.225\paperwidth]{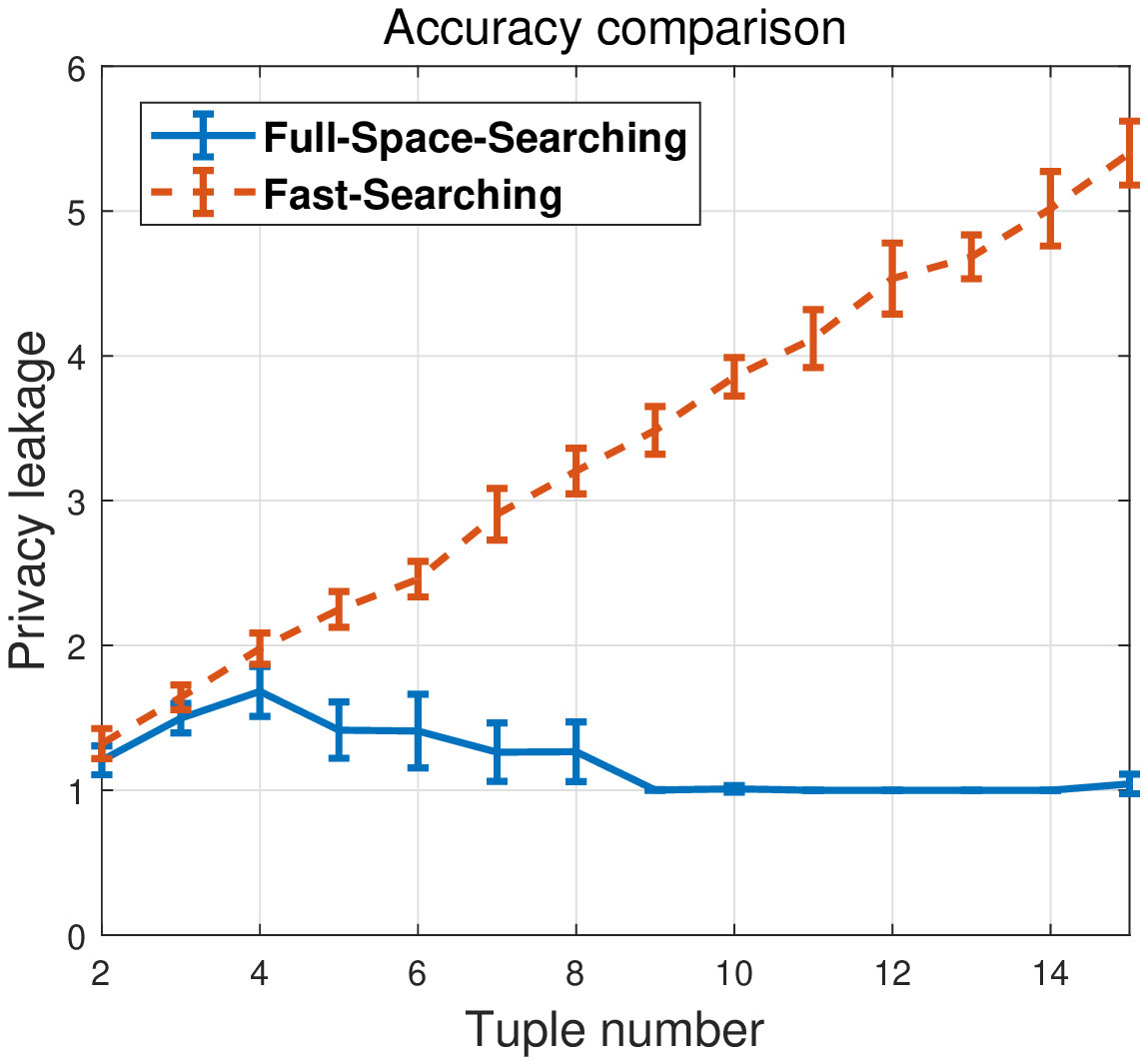}
\end{minipage}
\label{Fig:leakagecomparisonaverCorr0.2}
}
\subfigure[averCorr=0.5]{
\begin{minipage}[t]{0.23\linewidth}
\centering
\includegraphics[width=0.225\paperwidth]{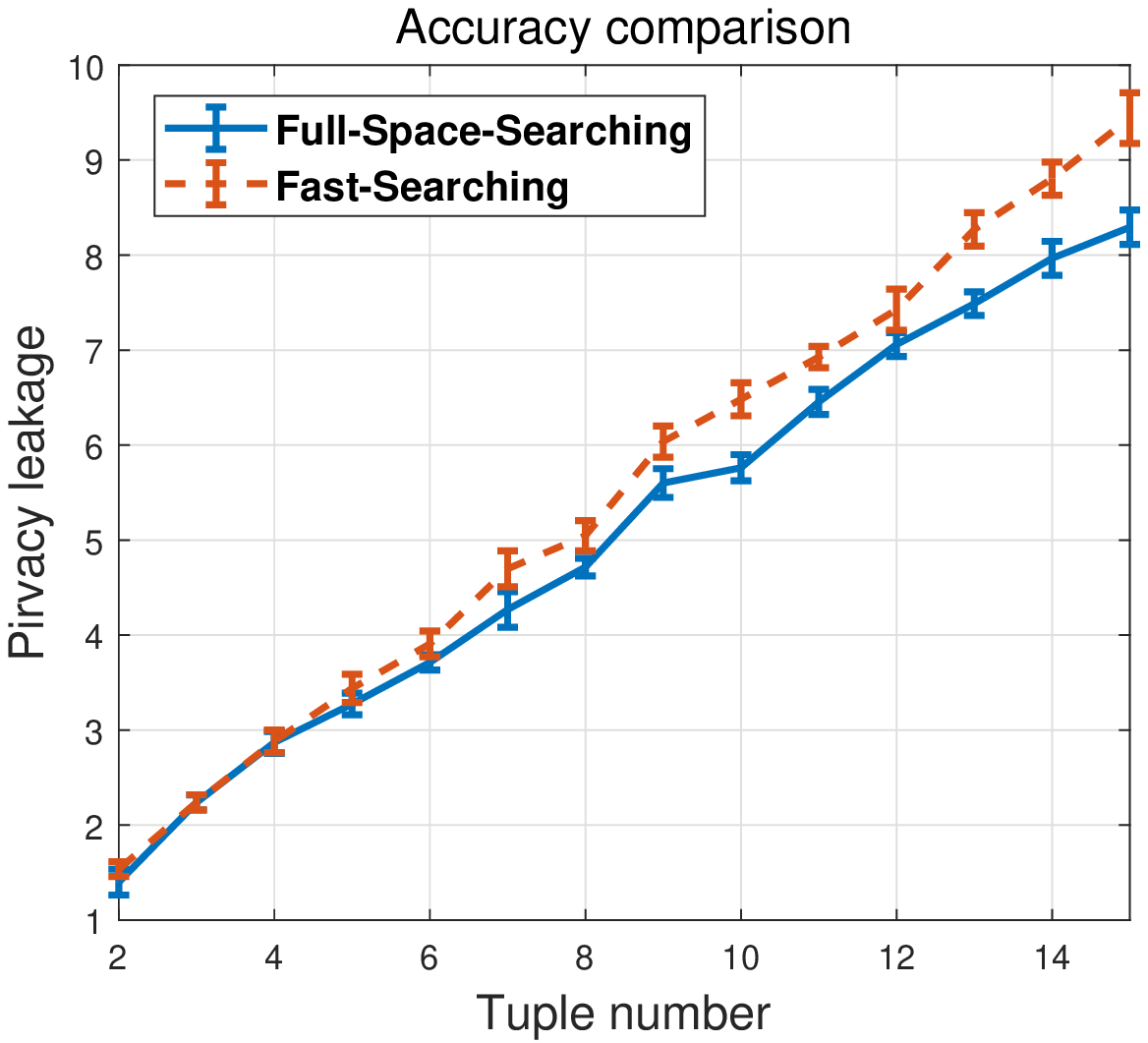}
\end{minipage}
\label{Fig:leakagecomparisonaverCorr0.5}
}
\subfigure[averCorr=0.8]{
\begin{minipage}[t]{0.23\linewidth}
\centering
\includegraphics[width=0.225\paperwidth]{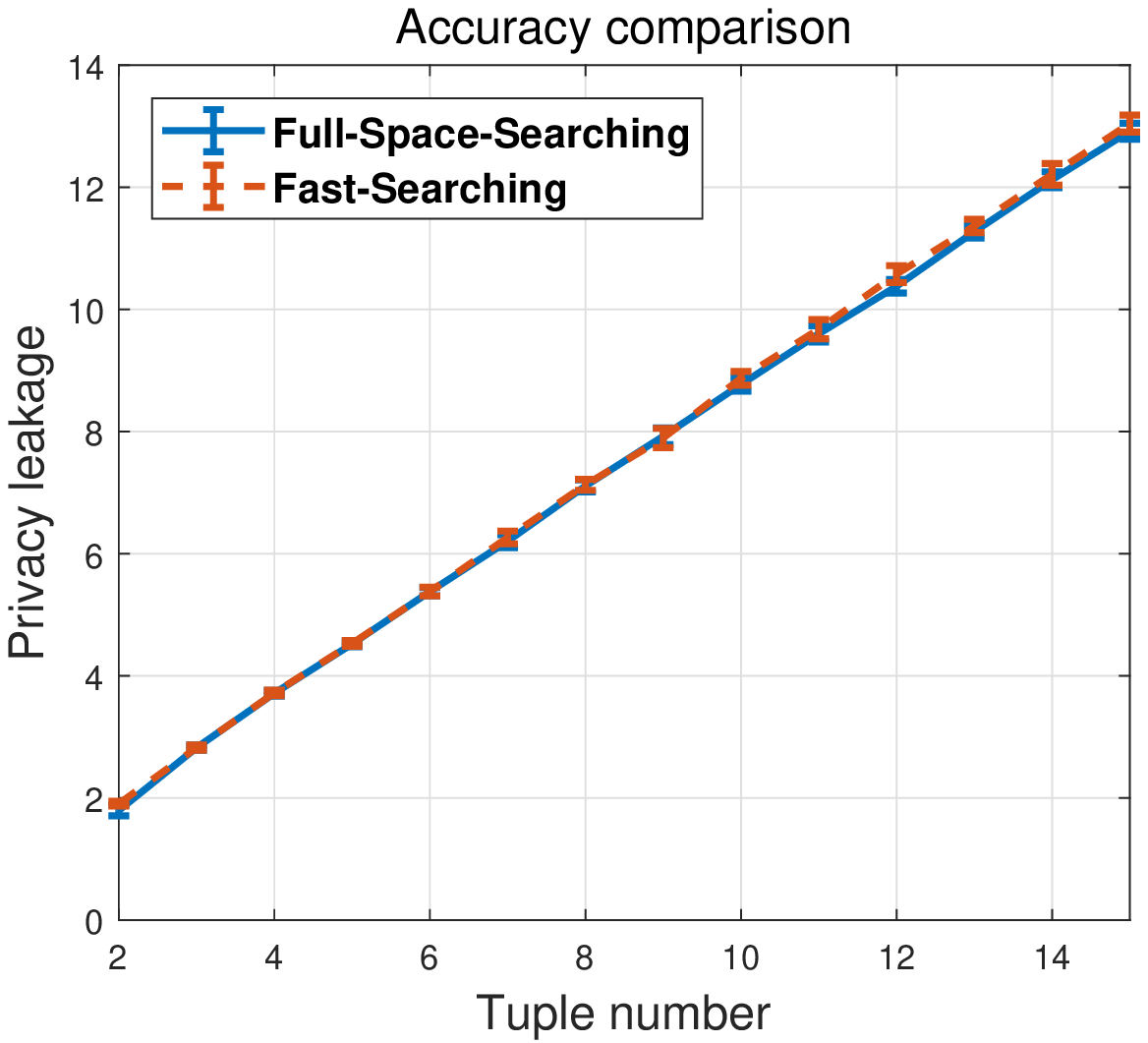}
\end{minipage}
\label{Fig:leakagecomparisonaverCorr0.8}
}
\subfigure[Average computation time]{
\begin{minipage}[t]{0.23\linewidth}
\centering
\includegraphics[width=0.225\paperwidth]{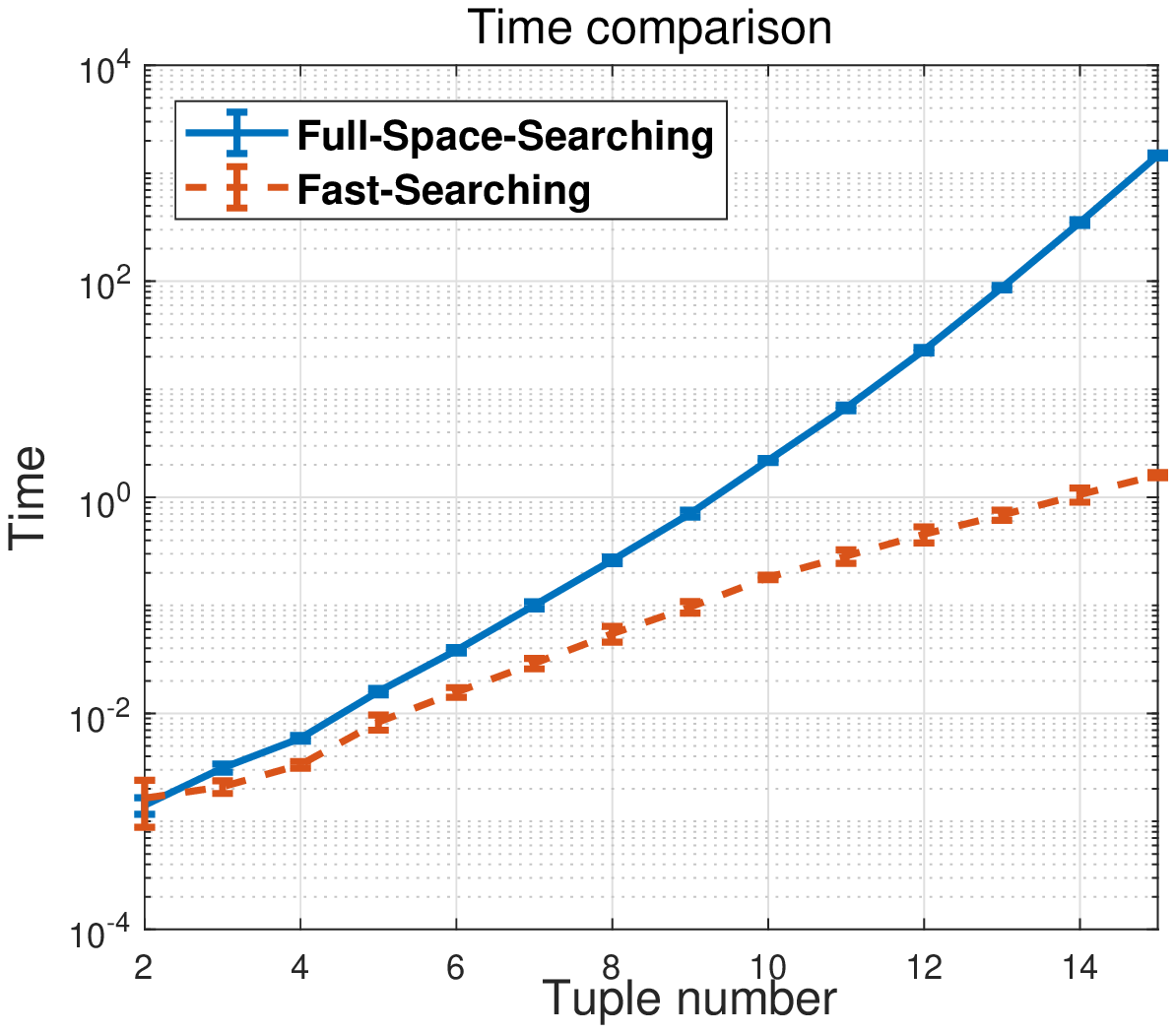}
\end{minipage}
\label{Fig:timecomparison}
}
\caption{The comparison of the full-space-searching algorithm and the fast-searching algorithm in terms of privacy leakage and computation time. Figs.\;\ref{Fig:leakagecomparisonaverCorr0.2}, \ref{Fig:leakagecomparisonaverCorr0.5} and \ref{Fig:leakagecomparisonaverCorr0.8} show the privacy leakage of both algorithms when the averCorr is 0.2, 0.5, and 0.8. Figs.\;\ref{Fig:timecomparison} shows the average computation time of both algorithms.}
\end{flushleft}
\label{fig:accuracyandtimecomparison}
\end{figure*}

 \subsubsection{Privacy Leakage vs Prior knowledge}

 This subsection investigates the impacts of prior knowledge on privacy leakage when the correlation is fixed.

 Figs.\;\ref{Fig:privacy discrite positive} and \ref{Fig:privacy continuous positive} show that the privacy leakage increases with the layer number for discrete data. That is, the privacy leakage decreases with the prior knowledge. This is because, given the positive correlation, more unknown tuples can cause a larger aggregation difference (Theorem \ref{TH:chain rule}) and less uncertainty for adversaries from the privacy-preserving results. In particular, the weakest adversary with the least prior knowledge will obtain the largest privacy gain, and thus leading to the highest privacy leakage (Corollary \ref{Cor:privacy leakage in sepcial cases}).

 Fig.\;\ref{Fig:privacy discrite negative} shows that there are no monotone trends between the privacy leakage and layer number when the data correlation is negative for discrete-valued data because tuples with mutually negative correlations will cancel each other out and show no general trend in the aggregation result, which makes it difficult for any adversary to achieve privacy gain. This corresponds to our analysis that the privacy leakage computed by the chain rule (Theorem \ref{TH:chain rule}) does not decrease with $IC_{ij,\mathcal{K}'}$ when $IC_{ij,\mathcal{K}'}<0$.
 Additionally, as we can see, the highest privacy leakage is achieved when the layer number is $1$ (the strongest prior knowledge), which is consistent with Corollary \ref{Cor:privacy leakage in sepcial cases}.

 Fig.\;\ref{Fig:privacy continuous negative} shows that privacy leakage decreases with the amount of prior knowledge for continuous data because the tuples with a mutually negative correlation will cancel each other out in the aggregation, which reduces the uncertainty of aggregation and makes it difficult for the adversary to infer an individual tuple. Specifically, based on the multivariate Gaussian model, more unknown tuples will lead to a stronger ``canceling'' effect and less privacy leakage for a weaker adversary.



 \subsubsection{Accuracy and Time Complexity}
 This subsection investigates the accuracy and efficiency of the fast-searching algorithm compared to the full-space-searching algorithm. In the simulation, we set the tuple number ranging from $1$ to $15$, and the average correlation from $0.2$ to $0.8$. Then, we computed the corresponding maximal privacy leakage in each case. 
 Each simulation was run 30 times; both the average privacy leakage and its variance were reported.

 Figs.\;\ref{Fig:leakagecomparisonaverCorr0.2}, \ref{Fig:leakagecomparisonaverCorr0.5}, and \ref{Fig:leakagecomparisonaverCorr0.8} compare the privacy leakage when averCorr equals 0.2, 0.5, and 0.8 respectively. The privacy leakage computed with the fast-searching algorithm was generally larger than that of the full-space-searching algorithm and led to overestimating the privacy since the search space of the fast-searching algorithm is a subset of that of the full-space-searching algorithm. However, when the average correlation was stronger (e.g., averCorr=0.5, 0.8), the privacy leakage computed with the fast-searching algorithm was very close to the accurate privacy leakage computed with the full-space-searching algorithm. However, the fast-searching algorithm was far more efficient than the full-space-searching algorithm. In particular, Fig.\;\ref{Fig:timecomparison} shows the comparison result of the average computational time for both algorithms. As we can see, the fast-searching algorithm with the time complexity of $O(n^4)$, required much less computational time than the full-space-searching algorithm with the time complexity of $O(n^42^{n-1})$.

 \section{Conclusion}
 \label{Sec:Conclusion}

In this paper, we present a unified analysis to investigate the impacts of general (positive, negative, and hybrid) data correlations and arbitrary prior knowledge possessed by adversaries on privacy leakage.
For continuous data, we obtain a closed-form expression of privacy leakage as a function of general data correlation and prior knowledge by using multivariate Gaussian distributions. For discrete data, a chain rule is derived to represent the privacy leakage, by using a WHG that can model the adversaries with arbitrary prior knowledge. All our analytical results are obtained by strictly mathematical proofs and hold for general linear queries. Numerical simulations validate our theoretical analysis. Future work will extend our analysis to nonlinear quires.


\bibliographystyle{IEEEtran}
\bibliography{conf}

\appendices

\section{Proof of Theorem \ref{TH:privacy of two adjacent nodes}}
\label{Proof:privacy of two nodes}

 \begin{proof}
   For two adversaries $\mathcal{A}_{i,\mathcal{K}}$ and its ancestor $\mathcal{A}_{i,\mathcal{K}'}$, $\mathcal{K}'=\mathcal{K}\backslash\{j\}$, by the law of total probability, we have
   \begin{align*}
     \mathrm{Pr}(r|x_i,\mathbf{x}_{\mathcal{K}'})=\sum_{x_j}\mathrm{Pr}(x_j|x_i,\mathbf{x}_{\mathcal{K}'})\mathrm{Pr}(r|x_i,\mathbf{x}_{\mathcal{K}}).
   \end{align*}
    Let $l_{\mathcal{A}_{i,\mathcal{K}}}$ denote the value of the node $(i,\mathcal{K})$. By the definition of PDP,
   $\sup_{x_i,x_i'}\log\frac{\mathrm{Pr}(r|x_i,\mathbf{x}_{\mathcal{K}})}{\mathrm{Pr}(r|x_i',\mathbf{x}_{\mathcal{K}})}\in [-l_{\mathcal{A}_{i,\mathcal{K}}},l_{\mathcal{A}_{i,\mathcal{K}}}]$. Therefore
   \begin{align}
     &l_{\mathcal{A}_{i,\mathcal{K}'}}=\sup_{x_i,x_i',r}\log\frac{\mathrm{Pr}(r|x_i,\mathbf{x}_{\mathcal{K}'})}{\mathrm{Pr}(r|x_i',\mathbf{x}_{\mathcal{K}'})}\notag\\
     &=\sup_{x_i,x_i',r}\log\frac{\sum_{x_j}\mathrm{Pr}(x_j|x_i,\mathbf{x}_{\mathcal{K}'})\mathrm{Pr}(r|x_i,\mathbf{x}_{\mathcal{K}})}
     {\sum_{x_j}\mathrm{Pr}(x_j|x_i',\mathbf{x}_{\mathcal{K}'})\mathrm{Pr}(r|x_i',\mathbf{x}_{\mathcal{K}})}\notag\\
     &\leq \sup_{x_i<x_i',r}\left|l_{\mathcal{A}_{i,\mathcal{K}}}+\log\frac{\sum_{x_j}\mathrm{Pr}(x_j|x_i,\mathbf{x}_{\mathcal{K}'})\mathrm{Pr}(r|x_i,\mathbf{x}_{\mathcal{K}})}
     {\sum_{x_j}\mathrm{Pr}(x_j|x_i',\mathbf{x}_{\mathcal{K}'})\mathrm{Pr}(r|x_i,\mathbf{x}_{\mathcal{K}})}\right|\notag\\
     & =\sup_{x_i<x_i'}\left|l_{\mathcal{A}_{i,\mathcal{K}}}+\log\frac{\sum_{x_j}\mathrm{Pr}(x_j|x_i,\mathbf{x}_{\mathcal{K}'})e^{{-x_j}/{\lambda}}}
     {\sum_{x_j}\mathrm{Pr}(x_j|x_i',\mathbf{x}_{\mathcal{K}'})e^{{-x_j}/{\lambda}}}\right|\label{Eq:laplace in twonodes proof}\\
     & =\left|l_{\mathcal{A}_{i,\mathcal{K}}}+IC_{ij,\mathcal{K}'}\right|.\label{Eq:LGR in twonodes proof}
   \end{align}
   Eq.\;(\ref{Eq:laplace in twonodes proof}) uses the Laplace mechanism and Eq.\;(\ref{Eq:LGR in twonodes proof}) is the definition of $IC_{ij,\mathcal{K}'}$.
 \end{proof}

\section{Proof of Theorem \ref{TH:edge and correlation}}
\label{Proof:LGR and correlation}

 To prove Theorem \ref{TH:edge and correlation}, we propose the next lemma, which is used to express the correlation by its conditional distribution.

 \begin{lem}\label{Lem:transformation of correlation}
  For a database $\mathbf{x}$ with two tuples $x_1=\{x_{1,1},x_{1,2}\}$, and $x_2=\{x_{2,1},x_{2,2}\}$. Let $y_i=\mathbb{E}(x_2|x_1=x_{1,i}),i=1,2,$ is the conditional expectation of $x_1$. The joint distribution is $p_{ij}=\Pr(x_1=x_{1,i},x_2=x_{2,j}), p_{i\cdot}=p_{i1}+p_{i2},i,j\in\{1,2\}$. Then, we have the next equivalent conditions of Pearson correlation coefficient of $x_1$ and $x_2$, denoted as $\rho_{12}$.
   \begin{equation}\label{Eq:relationship of rho and p}
     \begin{split}
       \rho_{12}>0\Leftrightarrow y_1< y_2\Leftrightarrow\frac{p_{11}}{p_{1\cdot}}>\frac{p_{21}}{p_{2\cdot}},\\
       \rho_{12}<0\Leftrightarrow y_1> y_2\Leftrightarrow\frac{p_{11}}{p_{1\cdot}}<\frac{p_{21}}{p_{2\cdot}},\\
       \rho_{12}=0\Leftrightarrow y_1= y_2\Leftrightarrow\frac{p_{11}}{p_{1\cdot}}=\frac{p_{21}}{p_{2\cdot}}.
     \end{split}
   \end{equation}
 \end{lem}

 \begin{proof}
 Based on the definition of the Pearson correlation coefficient, the plus-minus sign of $\rho_{12}$ is determined by its covariance $Cov(x_1,x_2)$. Using the properties of conditional expectation, we have
  \begin{align*}
 \mathbb{E}(x_1x_2)&=\mathbb{E}\{\mathbb{E}(x_1x_2|x_1)\}=x_{1,1} y_1p_{1\cdot}+x_{1,2} y_2p_{2\cdot},\\
 \mathbb{E}(x_2)&=\mathbb{E}\{\mathbb{E}(x_2|x_1)\}=y_1p_{1\cdot}+y_2p_{2\cdot}.
 \end{align*}
 Therefore, the covariance $Cov(x_1,x_2)$ can be written as
 \begin{align*}
   Cov(x_1,x_2)&=\mathbb{E}(x_1x_2)-\mathbb{E}(x_1)\mathbb{E}(x_2)\\
   &=x_{1,1} y_1p_{1\cdot}+x_{1,2} y_2p_{2\cdot}\\
   &\quad-(x_{1,1}p_{1\cdot}+x_{1,2}p_{2\cdot})\cdot(y_1p_{1\cdot}+y_2p_{2\cdot})\\
   &=(x_{1,2}-x_{1,1})(y_2-y_1).
 \end{align*}
 The last equation uses the fact that $p_{1\cdot}+p_{2\cdot}=1$. Note that the plus-minus sign of $\rho_{12}$ is equivalent to the sign of $Cov(x_1,x_2)$, then we prove the left half of Eq.\;(\ref{Eq:relationship of rho and p}) by setting $x_{1,2}>x_{1,1}$ as usual.

 Next, we prove the right half.
 Based on the definition of the conditional expectation of $y_i$, we have
 \begin{align}\label{Eq:left half relationship of rho and p}
  y_2-y_1= & x_{2,1}\frac{p_{21}}{p_{2\cdot}}+x_{2,2}\frac{p_{22}}{p_{2\cdot}}-x_{2,1}\frac{p_{11}}{p_{1\cdot}}-x_{2,2}\frac{p_{12}}{p_{1\cdot}}\notag\\
  =&(x_{2,2}-x_{2,1})\left(\frac{p_{11}}{p_{1\cdot}}-\frac{p_{21}}{p_{2\cdot}}\right).
 \end{align}
 Eq.\;(\ref{Eq:left half relationship of rho and p}) uses the facts that $\frac{p_{11}}{p_{1\cdot}}+\frac{p_{12}}{p_{1\cdot}}=1$, and $\frac{p_{21}}{p_{2\cdot}}+\frac{p_{22}}{p_{2\cdot}}=1$. Set $x_{2,2}>x_{2,1}$, then we complete the proof of the right half.
\end{proof}

 Proof of Theorem \ref{TH:edge and correlation}
 \begin{proof}
 (1) We prove that for any database $\mathbf{x}$, the value of $IR_{j,\mathcal{K}'}(x_{i,m},x_{i,n})$ is bounded in $[-1,1]$.

 Based on $\sum_{x_j}\mathrm{Pr}(x_j|x_i,\mathbf{x}_{\mathcal{K}'})=1$, we have
   \begin{align}\label{Eq:boundness of simplix of exp(-xj)}
    \min_{k} e^{\frac{-x_{j,k}}{\lambda}} \leq \sum_{x_j}\mathrm{Pr}(x_j|x_i,\mathbf{x}_{\mathcal{K}'})e^{\frac{-x_j}{\lambda}}\leq \max_{k} e^{\frac{-x_{j,k}}{\lambda}}.
   \end{align}
 Eq.\;(\ref{Eq:boundness of simplix of exp(-xj)}) holds for all $x_i$. We replace $x_i$ with two different values, $x_{i,m}$ and $x_{i,n}$ and have the following inequalities.
   \begin{align*}
    -\frac{LS_j(f)}{\lambda}\leq \log\frac{\sum_{x_j}\mathrm{Pr}(x_j|x_{i,m},\mathbf{x}_{\mathcal{K}'})e^{-{x_j}/{\lambda}}}{\sum_{x_j}\mathrm{Pr}(x_j|x_{i,n},\mathbf{x}_{\mathcal{K}'})e^{-{x_j}/{\lambda}}}
    \leq\frac{LS_j(f)}{\lambda}.
   \end{align*}
 Therefore, according to Eq.\;(\ref{Eq:definition of IR}), the definition of $IR_{j,\mathcal{K}'}(x_{i,m},x_{i,n})$, we have $IR_{j,\mathcal{K}'}(x_{i,m},x_{i,n})\in[-1,1]$.

 \vspace{1 ex}
 \noindent(2) For a database $\mathbf{x}$, two tuples among which are $x_i=\{x_{i,1},x_{i,2}\}$, and $x_j=\{x_{j,1},x_{j,2}\}$. The conditional joint distribution of $x_i$ and $x_j$ under $\mathbf{x}_{\mathcal{K}'}$ is $\Pr(x_i,x_j|\mathbf{x}_{\mathcal{K}'})$. We will prove that the correlations have a direct relation to $IR_{j,\mathcal{K}'}(x_{1,1},x_{1,2})$.
 According to Eq.\;(\ref{Eq:definition of IR}),
  \begin{align*}
    &IR_{j,\mathcal{K}'}(x_{i,1},x_{i,2})\\
    &=\log\frac{\sum_{k=1,2}\mathrm{Pr}(x_{j,k}|x_{i,1},\mathbf{x}_{\mathcal{K}'})e^{{-x_{j,k}}/{\lambda}}}
    {\sum_{k=1,2}\mathrm{Pr}(x_{j,k}|x_{i,2},\mathbf{x}_{\mathcal{K}'})e^{{-x_{j,k}}/{\lambda}}}\left/\frac{LS_j(f)}{\lambda}\right.
  \end{align*}
 Let
 \begin{align*}
   \mu_1=\frac{\Pr(x_{j,1}|x_{i,1},\mathbf{x}_{\mathcal{K}'})}{\sum_{k}\Pr(x_{j,k}|x_{i,1},\mathbf{x}_{\mathcal{K}'})},
   \nu_1=\frac{\Pr(x_{j,1}|x_{i,2},\mathbf{x}_{\mathcal{K}'})}{\sum_{k}\Pr(x_{j,k}|x_{i,2},\mathbf{x}_{\mathcal{K}'})}.
 \end{align*}
 Obviously, $\mu_1,\mu_2\in[0,1]$, and define the next function.
 \begin{align}
  f(\mu_1,\nu_1)=\log\frac{\mu_1e^{-{x_{j,1}}/{\lambda}}+(1-\mu_1)e^{-{x_{j,2}}/{\lambda}}}{\nu_1e^{-{x_{j,1}}/{\lambda}}
  +(1-\nu_1)e^{-{x_{j,2}}/{\lambda}}},
 \end{align}
 where the numerator and denominator are monotonically increasing with respect to $\mu_1$, and $\mu_2$, respectively.
 Based on these, we prove the three cases in Theorem \ref{TH:edge and correlation} by using Lemma \ref{Lem:transformation of correlation}.

 \noindent 1) If $\rho_{ij,\mathcal{K}'}>0$, by Lemma \ref{Lem:transformation of correlation}, $1\geq\mu_1>\nu_1\geq 0$. Therefore
 \begin{align*}
   \max{f(\mu_1,\nu_1)}&=f(1,0)=LS_j(f)/\lambda,\\
   \min{f(\mu_1,\nu_1)}&> f(a,a)=0, \forall a\in(0,1).
 \end{align*}
   So, $f(\mu_1,\nu_1)\in\left(0,LS_j(f)/{\lambda}\right]$, and $IR_{j,\mathcal{K}'}(x_{i,1},x_{i,2})\in(0,1]$.

 \noindent 2) If $\rho_{ij,\mathcal{K}'}<0$, by Lemma \ref{Lem:transformation of correlation}, $0\leq\mu_1<\nu_1\leq 1$. Therefore
  \begin{align*}
   \max{f(\mu_1,\nu_1)}&<f(a,a)=0,\forall a\in(0,1),\\
   \min{f(\mu_1,\nu_1)}&= f(0,1)=-LS_j(f)/\lambda.
 \end{align*}
  So, $f(\mu_1,\nu_1)\in\left[-LS_j(f)/{\lambda},0\right)$, and $IR_{j,\mathcal{K}'}(x_{i,1},x_{i,2})\in [-1,0).$

 \noindent 3) If $\rho_{ij,\mathcal{K}'}=0$, by Lemma \ref{Lem:transformation of correlation}, $\mu_1=\nu_1\in [0,1]$. Therefore, $f(\mu_1,\nu_1)\equiv 0$, and $IR_{j,\mathcal{K}'}(x_{i,1},x_{i,2})=0$.

\vspace{1 ex}
 \noindent (3) The conditions are the same as Case (2) except that $x_j=\{x_{j,1},\cdots,x_{j,s}\},s\geq3.$ Let $y_m=\mathbb{E}(x_j|x_{i,m},\mathbf{x}_\mathcal{K}'),m=1,2.$ We claim that the left half of Lemma \ref{Lem:transformation of correlation} holds without presenting the similar proof.

 Next, we prove Case (3). According to Eq.\;(\ref{Eq:definition of IR}),
  \begin{align*}
    &IR_{j,\mathcal{K}'}(x_{i,1},x_{i,2})\\
    &=\log\frac{\sum_{k=1}^s\mathrm{Pr}(x_{j,k}|x_{i,1},\mathbf{x}_{\mathcal{K}'})e^{{-x_{j,k}}/{\lambda}}}
    {\sum_{k=1}^s\mathrm{Pr}(x_{j,k}|x_{i,2},\mathbf{x}_{\mathcal{K}'})e^{{-x_{j,k}}/{\lambda}}}\left/\frac{LS_j(f)}{\lambda}\right..
  \end{align*}
 For $k=1,2,\cdots,s$, let
 \begin{align*}
   \mu_k=\frac{\Pr(x_{j,k}|x_{i,1},\mathbf{x}_{\mathcal{K}'})}{\sum_{k}\Pr(x_{j,k}|x_{i,1},\mathbf{x}_{\mathcal{K}'})},
   \nu_k=\frac{\Pr(x_{j,k}|x_{i,2},\mathbf{x}_{\mathcal{K}'})}{\sum_{k}\Pr(x_{j,k}|x_{i,2},\mathbf{x}_{\mathcal{K}'})}.
 \end{align*}
 Then, we have
  \begin{align*}
    &\sum_{k=1}^s\mathrm{Pr}(x_{j,k}|x_{i,1},\mathbf{x}_{\mathcal{K}'})e^{-\frac{x_{j,k}}{\lambda}}\\
    &=\sum_{k=1}^s \mu_k e^{-\frac{x_{j,k}}{\lambda}}\approx 1-\sum_{k=1}^s\mu_k \frac{x_{j,k}}{\lambda}.
  \end{align*}
   The last approximation is obtained by using $e^x\approx 1+x$ and the fact $\sum_{k}\mu_k=1$. Then, we get
  \begin{align}
    IR_{j,\mathcal{K}'}(x_{i,1},x_{i,2})\approx \left.\log\frac{1-\sum_{k}\mu_k{x_{j,k}}/{\lambda}}{1-\sum_{k}\nu_k{x_{j,k}}/{\lambda}}\right/\frac{LS_j(f)}{\lambda}.
  \end{align}
 With the additional condition $\lambda>GS(f)$, then we have $x_{j,k}/\lambda< 1,\forall k\in[s]$. Combining $\sum_{k}\mu_k=\sum_{k}\nu_k=1$,
 we obtain $\sum_{k}\mu_k{x_{j,k}}/{\lambda}<1$, and $\sum_{k}\nu_k{x_{j,k}}/{\lambda}<1$.
 Based on the extended expression of the left half of Lemma \ref{Lem:transformation of correlation}.
 We get
 \begin{equation*}
 \begin{split}
   \rho_{ij,\mathcal{K}'}>0 &\Leftrightarrow IR_{j,\mathcal{K}'}(x_{i,1},x_{i,2})\in(0,1],\\
   \rho_{ij,\mathcal{K}'}<0 &\Leftrightarrow IR_{j,\mathcal{K}'}(x_{i,1},x_{i,2})\in[-1,0),\\
   \rho_{ij,\mathcal{K}'}=0 &\Leftrightarrow IR_{j,\mathcal{K}'}(x_{i,1},x_{i,2}) =0.
 \end{split}
 \end{equation*}
 \end{proof}

\section{Proof of Theorem \ref{TH:privacy of Gaussian model}}
\label{Proof:privacy of Gaussian model}

Proof of Theorem \ref{TH:privacy of Gaussian model}

 \begin{proof}
 According to the PDP for a continuous-valued database, we compute the following
  \begin{align}\label{Eq:PDP with integration form}
    \frac{Pr(r|x_i,\mathbf{x}_{\mathcal{K}})}{Pr(r|x_i',\mathbf{x}_{\mathcal{K}})}
    =\frac{\int_{\mathbf{x}_\mathcal{U}}Pr(\mathbf{x}_{\mathcal{U}}|x_i,\mathbf{x}_{\mathcal{K}})Pr(r|s)\mathrm{d}\mathbf{x}_{\mathcal{U}}}
    {\int_{\mathbf{x}_\mathcal{U}}Pr(\mathbf{x}_{\mathcal{U}}|x_i',\mathbf{x}_{\mathcal{K}})Pr(r|s')\mathrm{d}\mathbf{x}_{\mathcal{U}}}
  \end{align}
 for any $\theta,r,|x_i-x_i'|\leq M$, where $s'=s_{\mathcal{U}}+x_i'+s_{\mathcal{K}}$.
 In accordance with Lemma \ref{Lem:property of Gaussian}, set $\mathbf{x}_u=\mathbf{x}_1|\mathbf{x}_2$, where $\mathbf{x}_1=\mathbf{x}_{\mathcal{U}}$, $\mathbf{x}_2=\{x_i,\mathbf{x}_{\mathcal{K}}\}$. Here, $u$ is the number of variables in $\mathbf{x}_{\mathcal{U}}$. According to Lemma \ref{Lem:property of Gaussian}, $\mathbf{x}_u$ follows $u$-dimensional Gaussian distribution, with the density function
    \begin{align}\label{Eq:density of Gaussian}
    f(\mathbf{x}_{u})=A\exp(-\frac{1}{2}(\mathbf{x}_{\mathcal{U}}-\boldsymbol{\mu}_{1|2})^\top\mathbf{\Sigma}_{1|2}^{-1}(\mathbf{x}_{\mathcal{U}}-\boldsymbol{\mu}_{1|2})),
  \end{align}
 where $A=(2\pi)^{-{u}/{2}}\left|\mathbf{\Sigma}_{1|2}\right|^{-{1}/{2}}$,
$\boldsymbol{\mu}_{1|2}=\boldsymbol{\mu}_1+\mathbf{\Sigma}_{12}\mathbf{\Sigma}_{22}^{-1}(\mathbf{x}_2-\boldsymbol{\mu}_2),
      \mathbf{\Sigma}_{1|2}=\mathbf{\Sigma}_{11}-\mathbf{\Sigma}_{12}\mathbf{\Sigma}_{22}^{-1}\mathbf{\Sigma}_{21}.$
 Because we adopt the Laplace mechanism,
 \begin{equation}\label{Eq:laplace in proof Gaussian}
  Pr(r|s)=\frac{1}{2\lambda}e^{-\frac{|r-s|}{\lambda}},
 \end{equation}
 where $s=s_\mathcal{U}+x_i+s_\mathcal{K}$ denotes the sum of unknown tuples, attack object tuple and known tuples.
 According to Lemma \ref{Lem:property of Gaussian}, $s_{\mathcal{U}}=\sum_{k\in\mathcal{U}}x_k$ follows the Gaussian distribution, i.e.,
 \begin{align}\label{Eq:sum U in proof of Gaussian}
  s_{\mathcal{U}}\sim N_1(\mu_0,\sigma_0^2)
 \end{align}
 where $\mu_0=\mathbf{1}^\top\cdot \boldsymbol{\mu}_{1|2}, \sigma_0^2=\mathbf{1}^\top\cdot \mathbf{\Sigma}_{1|2}\cdot \mathbf{1}.$

 By Eq.\;(\ref{Eq:variance of Gaussian}), $\sigma_0^2$ is a constant independent of $x_i$. By Eq.\;(\ref{Eq:mean of Gaussian}), $\mu_0$ has relation to $x_i$ and $\mathbf{x}_{\mathcal{K}}$. To analyze the influence of $x_i$, we should extract the item including $x_i$. Therefore, we expand $\mu_0$ and get
 \begin{equation}\label{Eq:expansion of mu_0 in Gaussian}
  \mu_0=\mu_{00}+\mu_{0i}x_i+\sum_{k\in\mathcal{K}}\mu_{0k}x_k,
 \end{equation}
 where $\mu_{00}$ is a symbol to represent that all items have no relation to $x_i,x_k$. Therefore, $\mu_0$ is only dependent on $x_i$ for given $x_k,k\in\mathcal{K}$. Combining Eq.\;(\ref{Eq:sum U in proof of Gaussian}) and Eq.\;(\ref{Eq:expansion of mu_0 in Gaussian}), the density function of $s_{\mathcal{U}}$ is
 \begin{align}
    f(s_{\mathcal{U}})&=\frac{1}{\sqrt{2\pi}\sigma_0}e^{-\frac{(s_{\mathcal{U}}-\mu_0)^2}{2\sigma_0^2}}\notag\\
    &=\frac{1}{\sqrt{2\pi}\sigma_0}e^{-\frac{(s_{\mathcal{U}}-\mu_{00}-\mu_{0i}x_i-\sum_{k\in\mathcal{K}}\mu_{0k}x_k)^2}{2\sigma_0^2}}.
    \label{Eq:density of sum X_U}
 \end{align}
 Let $z=r-s, t=r-s_{\mathcal{K}}-\mu_{00}-\sum_{k\in\mathcal{K}}\mu_{0k}x_k-(1+\mu_{0i})x_i$. Substituting Eq.\;(\ref{Eq:laplace in proof Gaussian}) and Eq.\;(\ref{Eq:density of sum X_U}) into Eq.\;(\ref{Eq:PDP with integration form}), we have
 \begin{align*}
   &\int_{\mathbf{x}_\mathcal{U}}Pr(\mathbf{x}_{\mathcal{U}}|x_i,\mathbf{x}_{\mathcal{K}})Pr(r|s)\mathrm{d}\mathbf{x}_{\mathcal{U}}=\\
   &\int_{z}\frac{1}{\sqrt{2\pi}\sigma_0}e^{-\frac{(t-z)^2}{2\sigma_0^2}}\frac{1}{2\lambda}e^{-\frac{|z|}{\lambda}}\mathrm{d}z
   =\frac{1}{2\lambda}e^{\frac{\sigma_0^2}{2\lambda^2}}G(\frac{t}{\lambda};\frac{\sigma_0}{\lambda}),\\
   &\int_{\mathbf{x}_\mathcal{U}}Pr(\mathbf{x}_{\mathcal{U}}|x_i',\mathbf{x}_{\mathcal{K}})Pr(r|s)\mathrm{d}\mathbf{x}_{\mathcal{U}}=\\
   &\int_{z}\frac{1}{\sqrt{2\pi}\sigma_0}e^{-\frac{(t'-z)^2}{2\sigma_0^2}}\frac{1}{2\lambda}e^{-\frac{|z|}{\lambda}}\mathrm{d}z
   =\frac{1}{2\lambda}e^{\frac{\sigma_0^2}{2\lambda^2}}G(\frac{t'}{\lambda};\frac{\sigma_0}{\lambda}),
 \end{align*}
 where $t'=r-s_{\mathcal{K}}-\mu_{00}-\sum_{k\in\mathcal{K}}\mu_{0k}x_k-(1+\mu_{0i})x_i'$. So
 \begin{align}
   \log\frac{Pr(r|x_i,\mathbf{x}_{\mathcal{K}})}{Pr(r|x_i',\mathbf{x}_{\mathcal{K}})}=\log G(\frac{t}{\lambda};\frac{\sigma_0}{\lambda})-\log G(\frac{t'}{\lambda};\frac{\sigma_0}{\lambda}).
 \end{align}
 By the mean value theorem and Lemma \ref{Lem:lemma for prove Gaussian model}, we further have
 \begin{align*}
   \log\frac{Pr(r|x_i,\mathbf{x}_{\mathcal{K}})}{Pr(r|x_i',\mathbf{x}_{\mathcal{K}})}&=\frac{\partial \log G(\xi)}{\partial (t/\lambda)}\cdot(\frac{t}{\lambda}-\frac{t'}{\lambda})\\
   &\leq \left| {t}/{\lambda}- {t'}/{\lambda}\right|=|1+\mu_{0i}|\cdot {|x_i-x_i'|}/{\lambda}.
 \end{align*}
 Under the assumption $|x_i-x_i'|\leq M$, the privacy leakage is
 \begin{align*}
   l_{\mathcal{A}_{i,\mathcal{K}}}(\theta)=|1+\mu_{0i}| {M}/{\lambda}.
 \end{align*}
\end{proof}

\ifCLASSOPTIONcaptionsoff
  \newpage
\fi

\end{document}